\documentclass[twoside,11pt]{article}

\usepackage{amsmath,empheq} 
\usepackage{mathtools}

\numberwithin{equation}{section}

\usepackage{jmlr2e}
\usepackage[english]{babel}
\usepackage{color}
\usepackage{bbm}
\usepackage{algorithm}
\usepackage{algorithmic}
\usepackage{verbatim}
\usepackage{hhline}
\usepackage{subcaption}

\usepackage{hyperref}

\newcommand{\explain}[2]{\underset{\mathclap{\overset{\uparrow}{#2}}}{#1}}

\def\P{\mathbb{P}}

\def\dim{\mathrm{dim}}

\def\cP{\mathcal{P}}
\def\cS{\mathcal{S}}

\jmlrheading{14}{2013}{xx-yy}{7/13}{??/??}{Yu-Xiang Wang and Huan Xu}

\ShortHeadings{Noisy Sparse Subspace Clustering}{Wang and Xu}
\firstpageno{1}

\begin{document}

\title{Noisy Sparse Subspace Clustering}

\author{\name Yu-Xiang Wang \email yuxiangw@cs.cmu.edu \\
       \addr Machine Learning Department\\
       Carnegie Mellon University\\
       Pittsburgh, PA 15213
       \AND
       \name Huan Xu \email mpexuh@nus.edu.sg\\
       \addr Department of Mechanical Engineering\\
       National University of Singapore\\
       Singapore 117576}

\editor{editors not assigned yet}

\maketitle

\begin{abstract}
This paper considers the problem of subspace clustering under noise. Specifically, we study the behavior of Sparse Subspace Clustering (SSC) when either adversarial or random noise is added to the unlabelled input data points, which are assumed to be in a union of low-dimensional subspaces.
We show that a modified version of SSC is \emph{provably effective} in correctly identifying the underlying subspaces, even with noisy data. This extends theoretical guarantee of this algorithm to more practical settings and provides justification to the success of SSC in a class of real applications.
\end{abstract}


\section{Introduction}

Subspace clustering is a problem motivated by many real applications. It is now widely known that many high dimensional data including motion trajectories~\citep{costeira1998motion_seg}, face images~\citep{basri2003lambertianface}, network hop counts~\citep{eriksson2011high_rankMC}, movie ratings~\citep{zhang2012RecSys} and social graphs~\citep{xu2011graphclustering} can be modelled as samples drawn from the {\em union} of multiple low-dimensional linear subspaces (illustrated in Figure~\ref{fig:Union_of_sub_model}). Subspace clustering, arguably the most crucial step to understand such data, refers to the task of clustering the data into their original subspaces and uncovers the underlying structure of the data. The partitions correspond to different rigid objects for motion trajectories, different people for face data, subnets for network data, like-minded users in movie database and latent communities for social graph.

\begin{figure}
  \centering
  \includegraphics[width=0.7\linewidth]{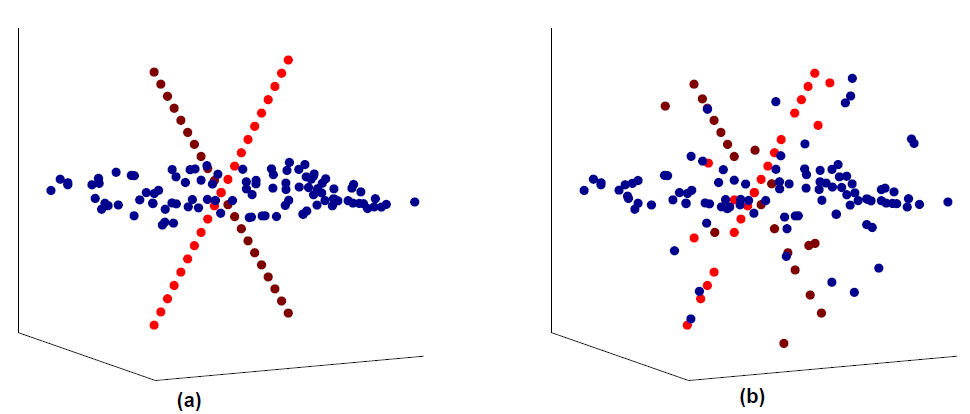}\\
  \caption{Exact (a) and noisy (b) data in union-of-subspace}\label{fig:Union_of_sub_model}
\end{figure}

Subspace clustering has drawn significant attention in the last decade and a great number of algorithms have been proposed, including Expectation-Maximization-like local optimization algorithms, e.g., K-plane~\citep{bradley2000k-plane} and Q-flat~\citep{tseng2000qflat}, algebraic methods, e.g., Generalized Principal Component Analysis (GPCA)~\citep{vidal2005gpca}, matrix factorization methods~\citep{costeira1998motion_seg,costeira2000multibody_factorization}, spectral clustering-based methods~\citep{lauer2009spectral,chen2009spectral}, bottom-up local sampling and affinity-based methods \citep[e.g.,][]{yan2006LSA,rao2008motion}, and the convex optimization-based methods: namely, Low Rank Representation (LRR)~\citep{liu2010lrr_icml,liu2013LRR} and Sparse Subspace Clustering (SSC)~\citep{elhamifar2009ssc,elhamifar2012ssc_journal}. For a comprehensive survey and comparisons, we refer the readers to the tutorial~\citep{vidal2011tutorial}. Among these algorithms, SSC is known to enjoy superb empirical performance, {\em even for noisy data}. For example, it
is the state-of-the-art algorithm for motion segmentation on Hopkins155 benchmark~\citep{tron2007benchmark,elhamifar2009ssc}, and has been shown to be more robust than LRR as the number of clusters increase~\citep{elhamifar2012ssc_journal}.

The key idea of SSC is to represent each data point by a sparse linear combination of the remaining data points using $\ell_1$ minimization. Without introducing the notations (which is deferred in Section~\ref{sec:prob_setup}), the noiseless and noisy version of SSC solve respectively
\begin{align*}
  \min_{c_i} \|c_i\|_1\quad s.t.\quad x_i=X_{-i}c_i, && \text{and}&& \min_{c_i} \|c_i\|_1+\frac{\lambda}{2}\|x_i-X_{-i}c_i\|^2,
\end{align*}
for each data column $x_i$, and the hope is that $c_i$ will be supported only on indices of the data points from the same subspace as $x_i$. While this formulation is for linear subspaces, affine subspaces can also be dealt with by augmenting data points with an offset variable $1$.

Effort has been made to explain the practical success of SSC by analyzing the noiseless version. \citet{elhamifar2010ssc_icassp} show that under certain conditions, \emph{disjoint} subspaces (i.e., they are not overlapping) can be exactly recovered.
A recent geometric analysis of SSC~\citep{soltanolkotabi2011geometric} broadens the scope of the results significantly to the case when subspaces can be overlapping. However, while these analyses advanced our understanding of SSC, one common drawback
is that data points are assumed to be lying {\em exactly} on the subspaces. This assumption can hardly be satisfied in practice. For example, motion trajectories data are only {\em approximately} of rank-4 due to perspective distortion of camera, tracking errors and pixel quantization~\citep{costeira1998motion_seg}; similary, face images are   not precisely of rank-9 since human faces are at best {\em approximated} by a convex body~\citep{basri2003lambertianface}.

In this paper, we address this problem and provide a theoretical analysis of SSC with noisy or corrupted data. Our main result shows that a modified version of SSC (see Eq. \eqref{eq:Lasso}) succeeds when the magnitude of noise does not exceed a threshold determined by a geometric gap between the \emph{inradius} and the \emph{subspace incoherence} (see below for precise definitions). This complements the result of \citet{soltanolkotabi2011geometric} that shows the same geometric gap determines whether SSC succeeds for the noiseless case. Indeed,  when the noise vanishes, our results reduce to the noiseless case results of \citeauthor{soltanolkotabi2011geometric}.

While our analysis is based upon the geometric analysis of \citet{soltanolkotabi2011geometric}, the analysis is more involved: In SSC, sample points are used as the dictionary for sparse recovery, and therefore noisy SSC requires analyzing a noisy dictionary.
We also remark that our results on noisy SSC are {\em exact}, i.e., as long as the noise magnitude is smaller than a threshold, the recovered subspace clusters are {\em correct}.
This is in sharp contrast to the majority of previous work on structure recovery for noisy data where stability/perturbation bounds are given~--~i.e., the obtained solution is {\em approximately} correct, and the approximation gap goes to zero only when the noise diminishes.

Lastly, we remark that an independently developed work \citep{soltanolkotabi2013robust} analyzed the same algorithm {\em under a statistical model} that generates the data. In contrast, our main results focus on the cases when the data are deterministic. Moreover, when we specialize our general result to the same statistical model, we show that we can handle a significantly larger amount of noise under certain regimes.


The paper is organized as follows. In Section~\ref{sec:RelatedWorks}, we review previous and ongoing works related to this paper. In Section~\ref{sec:prob_setup}, we formally define the notations, explain our method and the models of our analysis. Then we present our main theoretical results in Section~\ref{sec:main} with examples and remarks to explain the practical implications of each theorem. In Section~\ref{sec:proof_deterministic}~and~\ref{sec:proof_randomized}, proofs of the deterministic and randomized results are provided. We then evaluate our method experimentally in Section~\ref{sec:experiments} with both synthetic data and real-life data, which confirms the prediction of the theoretical results. Lastly, Section~\ref{sec:conclusion} summarizes the paper and discuss some open problems for future research in the task of subspace clustering.


\section{Related works}\label{sec:RelatedWorks}
In this section, we review  previous and ongoing theoretical studies on the problem of subspace clustering.

\subsection{Nominal performance guarantee for noiseless data}
Most previous analyses concern about the nominal performance of a particular subspace clustering algorithm with noiseless data. The focus is to relax the assumptions on the underlying subspaces and data generation.

A number of methods have been shown working under the \emph{independent subspace} assumption including the early factorization-based methods \citep{costeira1998motion_seg,kanatani2001motion}, LRR~\citep{liu2010lrr_icml} and the initial guarantee of SSC~\citep{elhamifar2009ssc}. Recall that the data points are drawn from a union of subspaces, the \emph{independent subspace } assumption requires each subspace to be linearly independent to the {\em span} of all other subspaces. Equivalently, this assumption requires the sum of each subspace's dimension to be equal to the dimension of the span of all subspaces. For example, in a two dimensional plane, one can only have 2 independent lines. If there are three lines intersecting at the origin, even if each pair of the lines are independent, they are not considered independent as a whole.

\emph{Disjoint subspace} assumption only requires pairwise linear independence, and hence is more meaningful in practice. To the best of our knowledge, only GPCA~\citep{vidal2005gpca} and SSC~\citep{elhamifar2010ssc_icassp,elhamifar2012ssc_journal} have been shown to provably handle the data under \emph{disjoint subspace} assumption. GPCA however is not a polynomial time algorithm. Its computational complexity increases exponentially with respect to the number and dimension of the subspaces.

\citet{soltanolkotabi2011geometric} developed a geometric analysis that further extends the performance guarantee of SSC, and in particular it covers some cases when the underlying subspaces are slightly \emph{overlapping}, meaning that two subspaces can even share a basis. The analysis reveals that the success of SSC relies on the difference of two geometric quantities (inradius $r$ and incoherence $\mu$) to be greater than $0$, which leads to by far the most general and strongest theoretical guarantee for noiseless SSC. A summary of these assumptions on the subspaces and their formal definition are given in Table~\ref{tab:subspaces}.
\begin{table}
  \centering
\begin{tabular}{|l|c|}
  \hline
  Independent Subspaces & $\dim\left[\cS_1\otimes...\otimes \cS_L\right] = \sum_{\ell=1}^L \dim\left[\cS_\ell\right]  $.  \\\hline
  Disjoint Subspaces &  $\cS_\ell\cap \cS_k =\mathbf{0}$ for all $\{(\ell,k)|\ell\neq k\}$.\\\hline
  Overlapping Subspaces & No points lies in $\cS_\ell\cap \cS_k$ for any $\{(\ell,k)|\ell\neq k\}$.\\
  \hline
\end{tabular}
\caption{Comparison of conditions on the underlying subspaces.}\label{tab:subspaces}
\end{table}

We remark that our robust analysis extends from \citet{soltanolkotabi2011geometric} and therefore is inherently capable of handling the same range of problems, namely disjoint and overlapping subspaces. This is formalized later in Section~\ref{sec:main}.

\subsection{Robust performance guarantee}

Previous studies of the subspace clustering under noise have been mostly empirical. For instance, factorization, spectral clustering and local affinity based approaches, which we mentioned above, are able to produce a (sometimes good) solution even for noisy real data. Convex optimization based approaches like LRR and SSC can be naturally reformulated as a robust method by relaxing the hard equality constraints to a penalty term in the objective function. In fact, the superior results of SSC and LRR on motion segmentation and face clustering data are produced using the robust extension in \citet{elhamifar2009ssc} and \citet{liu2010lrr_icml} instead of the well-studied noiseless version.

As of writing, there have been very few subspace clustering methods that is guaranteed to work when data are noisy.  Besides the conference version of the current paper~\citep{wang2013noisy}, an independent work \citep{soltanolkotabi2013robust} also analyzed SSC under noise. Subsequently, there has been noisy guarantees for other algorithms, e.g., thresholding based approach \citep{heckel2013noisy} and
orthogonal matching pursuit \citep{dyer2013greedy}.

The main difference between our work and \citep{soltanolkotabi2013robust} is that our guarantee works for a more general set of problems when the data and noise may not be random, whereas the key arguments in the proof in \citet{soltanolkotabi2013robust} relies on the assumption that data points are uniformly distributed on the unit sphere within each subspace, which corresponds to the ``semi-random model'' in our paper.
As illustrated in \citet[Figure~9~and~10]{elhamifar2012ssc_journal}, the semi-random model is not a good fit for both the motion segmentation and the face clustering datasets, as in these datasets there is a fast decay in the singular values of each subspace.  The uniform distribution assumption becomes even harder to justify as the dimension $d$ of each subspace gets larger --- a regime where the analysis in \citep{soltanolkotabi2013robust} focuses on.

Moreover, with a minor modification in our analysis that sharpens the bound of the tuning parameter that ensures the solution is non-trivial, we are able to get a result that is stronger than \citet{soltanolkotabi2013robust} in cases when the dimension of each subspace $d\leq O(\sqrt{n})$ \footnote{Admittedly, \citep{soltanolkotabi2013robust} obtained better noise-tolerance than the comparable result in our conference version \citep{wang2013noisy}. }. This result extends the provably guarantee of SSC to a setting where the signal to noise ratio (SnR) is allowed to go to $0$ as the ambient dimension gets large. In summary, we compare our results in terms of the level of noise that can be provably tolerated in Table~\ref{tab:comparison}. These comparisons are in the same setting modulo some slight differences in the noise model and successful criteria. It is worth noting that when $d>O(\sqrt{n})$, \citet{soltanolkotabi2013robust}'s bound is sharper. We will provide more details in the Appendix.



\begin{table}
\centering
\small{
\begin{tabular}{|p{2.2in}|p{1.05in}|p{1in}|p{1in}|p{1.2in}|}
  \hline
   & This paper & \citep{wang2013noisy} & \citet{soltanolkotabi2013robust} \\\hline 
  Fully deterministic & $O(r(r-\mu))$ & $O(r(r-\mu))$ & N.A.  \\\hline
  Deterministic + random noise & $O((n/d)^{\frac{1}{4}}(r-\mu))$ & $O(r-\mu)$ & N.A.   \\\hline
  Semi-random data + random noise & $O\left(\frac{n^{\frac{1}{4}}}{\sqrt{d}}(1-\frac{\text{aff}}{\sqrt{d}})\right)$ & $O\left(\frac{1}{\sqrt{d}}(1-\frac{\text{aff}}{\sqrt{d}})\right)$ & $O\left(1-\frac{\text{aff}}{\sqrt{d}}\right)$  \\\hline
    Fully-random data + random noise & $O\left(\frac{n^{\frac{1}{4}}}{\sqrt{d}}(1-\frac{\sqrt{d}}{\sqrt{n}})\right)$ & $O\left(\frac{1}{\sqrt{d}}(1-\frac{\sqrt{d}}{\sqrt{n}})\right)$ & $O\left(1-\frac{\sqrt{d}}{\sqrt{n}}\right)$  \\
  \hline
\end{tabular}
}
\caption{Comparison of the level of noise tolerable for noisy subspace clustering methods. Note that ``$\mathrm{aff}$'' is the ``unnormalized'' affinity defined in \citep{soltanolkotabi2011geometric}}.\label{tab:comparison}
\end{table}

Lastly, we note that the notion of robustness in this paper is confined to the noise/arbitrary corruptions added to the legitimate data. It is not the robustness against outliers in the data, unless otherwise specified. Handling outliers is a completely different problem. Solutions have been proposed for LRR in \citet{liu2012aistats} by decomposing a $\ell_{2,1}$ norm column-wise sparse components and for SSC in \citet{soltanolkotabi2011geometric} by objective value thresholding. However these results require non-outlier data points to be free of noise, therefore are not comparable to the study in this paper.

\section{Problem setup}\label{sec:prob_setup}
\paragraph{Notations: }
We denote the uncorrupted data matrix by $Y \in \mathbb{R}^{n\times N}$, where each column of $Y$ (normalized to unit vector \footnote{We assume the normalization condition for ease of presentation. Our results can be extened to the case when   each column of the noisy data points $X=Y+Z$ is normalized, as well as the case where no normalizing is performed at all, under simple modifications to the conditions.   }) belongs to a union of $L$ subspaces $$\mathcal{S}_1 \cup \mathcal{S}_2 \cup...\cup \mathcal{S}_L.$$

Each subspace $\mathcal{S}_{\ell}$ is of dimension $d_{\ell}$ and contains $N_{\ell}$ data samples with $N_1 +N_2+...+N_L=N$. We observe the noisy data matrix $X = Y+Z$, where $Z$ is some arbitrary noise matrix. Let $Y^{(\ell)}\in \mathbb{R}^{n\times N_{\ell}}$ denote the selection of columns in $Y$ that belongs to $\mathcal{S}_{\ell}$, and denote the corresponding columns in $X$ and $Z$  by $X^{(\ell)}$ and $Z^{(\ell)}$ respectively. Without loss of generality, let $X=[X^{(1)},X^{(2)},...,X^{(L)}]$ be ordered. In addition, we use subscript ``$-i$'' to represent a matrix that excludes column~$i$, e.g., $X^{(\ell)}_{-i}=[x^{(\ell)}_1,...,x^{(\ell)}_{i-1},x^{(\ell)}_{i+1},...,x^{(\ell)}_{N_{\ell}}].$ Calligraphic letters such as $\mathcal{X},\mathcal{Y_{\ell}}$ represent the set containing all columns of the corresponding matrix (e.g., $X$ and $Y^{(\ell)}$).

For any matrix $X$, $\mathcal{P}(X)$ represents the symmetrized convex hull of its columns, i.e., $\mathcal{P}(X) = \mathrm{conv}(\pm \mathcal{X})$. Also let $\mathcal{P}_{-i}^{(\ell)} := \mathcal{P}(X_{-i}^{(\ell)})$ and $\mathcal{Q}_{-i}^{(\ell)} := \mathcal{P}(Y_{-i}^{(\ell)})$ for short. $\mathbb{P}_{\mathcal{S}}$ and $\mathrm{Proj}_{\mathcal{S}}$ denote respectively the projection matrix and projection operator (acting on a set) to subspace $\mathcal{S}$. Throughout the paper, $\|\cdot\|$ represents $2$-norm for vectors and operator norm for matrices; other norms will be explicitly specified (e.g., $\|\cdot\|_1,\|\cdot\|_{\infty}$).

\paragraph{Method: }
Original SSC solves the linear program
    \begin{equation}\label{eq:SSC}
    \begin{aligned}
    \min_{c_i} \; \|c_i\|_1 \quad s.t. \quad &x_i=X_{-i}c_i,
    \end{aligned}
    \end{equation}
for each data point $x_i$. Solutions are arranged into matrix $C=[c_1,...,c_N]$, then spectral clustering techniques such as \citet{ng2002spectral} are applied on the affinity matrix $W=|C|+|C|^T$ ($|\cdot|$ represents entrywise absolute value). Note that when $Z\neq 0$, this method breaks down: indeed \eqref{eq:SSC} may even be infeasible.


To handle noisy $X$, a natural extension is to relax the equality constraint in \eqref{eq:SSC} and solve the following unconstrained minimization problem instead \citep{elhamifar2012ssc_journal}:
\begin{equation}\label{eq:Lasso}
\begin{aligned}
\min_{c_i} \; &\|c_i\|_1+\frac{\lambda}{2}\|x_i-X_{-i}c_i\|^2.
\end{aligned}
\end{equation}
We will focus on Formulation~\eqref{eq:Lasso} in this paper. Notice that \eqref{eq:Lasso} coincides with standard LASSO. Yet, since our task is subspace clustering, the analysis of LASSO (mainly for the task of support recovery) does not extend to SSC. In particular, existing literature for LASSO to succeed requires the dictionary $X_{-i}$ to satisfy the Restricted Isometry Property~\citep[RIP for short;][]{candes2008RIP} or the Null-space property~\citep{donoho2006BPDN},  but neither of them is satisfied in the subspace clustering setup.\footnote{As a simple illustrative example, suppose there exists two identical columns in $X_{-i}$, which violates RIP for 2-sparse signal and has maximum incoherence $\mu(X_{-i})=1$.}

In the subspace clustering task, there is no single ``ground-truth'' $C$ to compare the solution against. Instead, the algorithm succeeds if each sample is expressed as a linear combination of samples belonging to the same subspace, as the following definition states.
\begin{definition}[LASSO Subspace Detection Property]\label{def:lasso_detection}
We say the subspaces $\{\mathcal{S}_{\ell}\}_{\ell=1}^{k}$ and noisy sample points $X$ from these subspaces obey LASSO subspace detection property with parameter $\lambda$, if and only if it holds that for all $i$, the optimal solution $c_i$ to \eqref{eq:Lasso} with parameter $\lambda$ satisfies:\\
\indent (1) $c_i$ is not a zero vector, i.e., the solution is non-trivial,
\indent (2) Nonzero entries of $c_i$ correspond to only columns of $X$ sampled from the same subspace as $x_i$.
\end{definition}
This property ensures that the output matrix $C$ and (naturally) the affinity matrix $W$ are exactly block diagonal with each subspace cluster represented by a disjoint block.  The property is illustrated in Figure~\ref{fig:SEP}. For convenience, we will refer to the second requirement alone as ``\emph{Self-Expressiveness Property}''~(SEP), as defined in \citet{elhamifar2012ssc_journal}.

Note that the LASSO Subspace Detection Property is a strong condition. In practice, spectral clustering does not require the exact block diagonal structure for perfect segmentation (check Figure~\ref{fig:Exp1_acc_map} in our simulation section for details). A caveat is that it is also not sufficient for perfect segmentation, since it does not guarantee each diagonal block forms a connected component. This is a known problem for SSC \citep{nasihatkon2011graph}, although we observe that in practice graph connectivity is usually not a big issue. Proving the high-confidence connectivity (even under probabilistic models) remains an open problem, except for the almost trivial cases when the subspaces are independent \citep{liu2013LRR, wang2013provable}.

\begin{figure}
  \centering
  \includegraphics[width=0.35\linewidth]{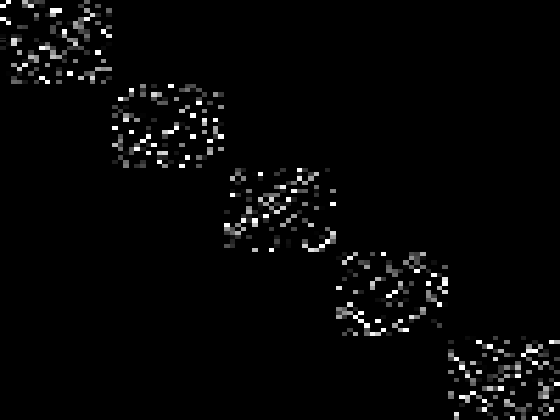}
  \includegraphics[width=0.35\linewidth]{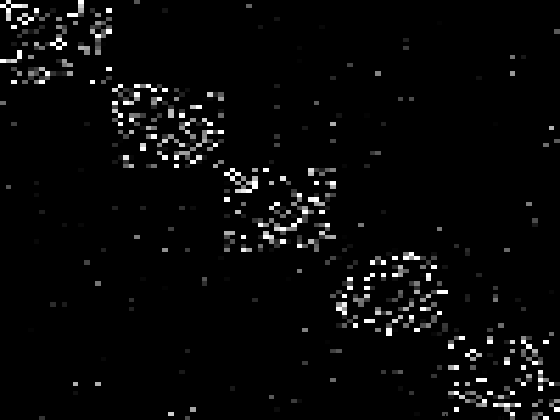}\\
  \caption{Illustration of LASSO-Subspace Detection Property/Self-Expressiveness Property. \textbf{Left:} SEP holds. \textbf{Right:} SEP is violated even though spectral clustering is likely to cluster this affinity graph perfectly into 5 blocks.}\label{fig:SEP}
\end{figure}

\paragraph{Models of analysis: }
Our objective here is to provide sufficient conditions upon which the LASSO subspace detection properties hold in the following four models. Precise definition of the noise models will be given in Section~\ref{sec:main}.

\begin{tabular}{ll}
  $\bullet$ fully deterministic model;\\
  $\bullet$ deterministic data with random noise;\\
  $\bullet$ semi-random data with random noise;\\
  $\bullet$ fully random model.
\end{tabular}

\section{Main results}\label{sec:main}
\subsection{Deterministic model}
We start by defining two concepts adapted from the original proposal of \citet{soltanolkotabi2011geometric}.
\begin{definition}[Projected Dual Direction]\label{def:proj_dual_direction}
Let $\nu$ be the optimal solution to the dual optimization program\footnote{This definition is related to \eqref{eq:Opt_original_dual}, the dual problem of \eqref{eq:Lasso}, which we will define in the proof.}
\begin{align*}
\quad \max_{\nu} \; \langle x,\nu \rangle - \frac{1}{2\lambda}\nu^T\nu,\quad\text{subject to:}\quad &\|X^T\nu\|_{\infty} \leq 1;
\end{align*}
and $\mathcal{S}$ is a low-dimensional subspace. The {\em projected dual direction} $v$ is  defined as
$$v(x,X,\mathcal{S},\lambda)\triangleq\frac{\mathbb{P}_{\mathcal{S}} \nu}{\|\mathbb{P}_{\mathcal{S}} \nu\|}.$$
\end{definition}

\begin{definition}[Projected Subspace Incoherence Property]\label{def:incoherence}
Compactly denote projected dual direction $v_i^{(\ell)}=v(x_i^{(\ell)},X_{-i}^{(\ell)},\mathcal{S}_{\ell},\lambda)$ and $V^{(\ell)}=[v_1^{(\ell)},...,v_{N_{\ell}}^{(\ell)}]$. We say that vector set $\mathcal{X}_{\ell}$ is $\mu$-incoherent to other points if
\begin{align*}
    \mu\geq \mu(\mathcal{X}_{\ell}) := &\max_{y\in \mathcal{Y}\setminus \mathcal{Y}_{\ell}}{\|{V^{(\ell)}}^Ty\|_{\infty}}.
\end{align*}
\end{definition}
%

Here, $\mu$ measures the incoherence between corrupted subspace samples $\mathcal{X}_{\ell}$ and clean data points in other subspaces (illustrated in Figure~\ref{fig:SubspaceIncoherence}). As $\|y\|=1$ by the normalization assumption, the range of $\mu$ is $[0,1]$. In case of random subspaces in high dimension, $\mu$ is close to zero. Moreover, as we will see later, for deterministic subspaces and random data points, $\mu$ is proportional to their expected angular distance (measured by cosine of canonical angles).

Definition~\ref{def:proj_dual_direction}~and~\ref{def:incoherence} differ from the \emph{dual direction} and \emph{subspace incoherence property} of \citet{soltanolkotabi2011geometric} in that we require a projection to a particular subspace to cater to the analysis of the noise case. Also, since they reduce to the original definitions when data are noiseless and $\lambda\rightarrow \infty$, these definitions can be considered as a generalization of their original version.

%
\begin{figure}
  \centering
      \includegraphics[width=0.7\linewidth]{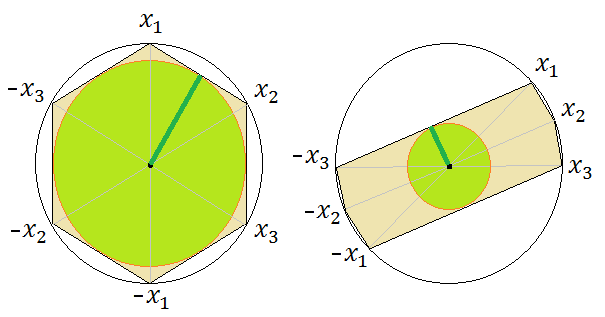}\\
  \caption{Illustration of inradius and data distribution. The inradius measures how well data points represent a subspace. }\label{fig:inradius}
  \includegraphics[width=0.9\linewidth]{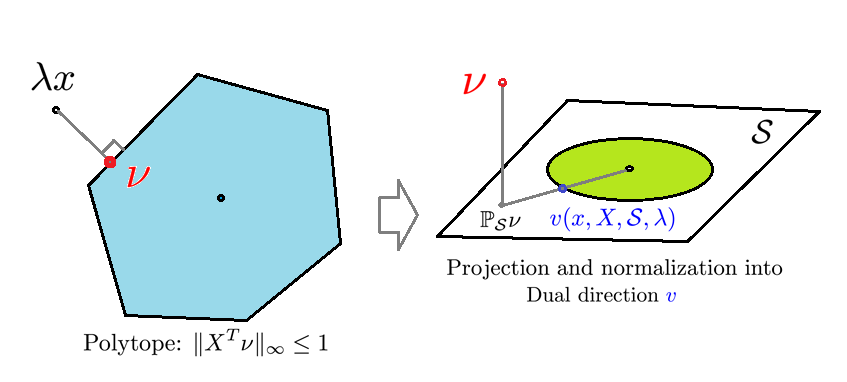}\\
  \includegraphics[width=0.9\linewidth]{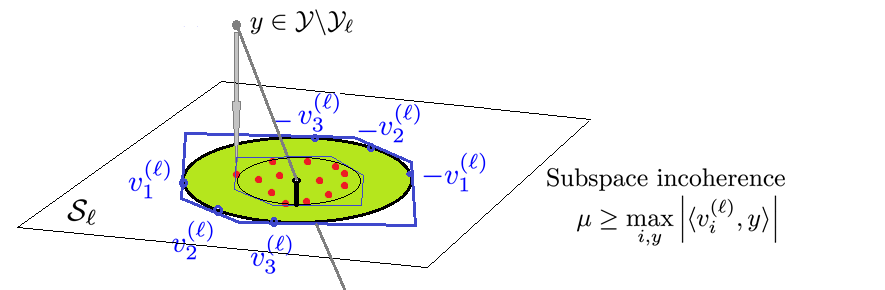}
  \caption{Illustrations of the projected dual direction and subspace incoherence property. The projected dual direction in Definition~\ref{def:proj_dual_direction} is essentially an Euclidean projection to the polytope, followed by a projection to the subspace and normalization. There is a dual direction associated with each data point in the subspace. Jointly, $\left\{x \middle| \max_{i}\big|\langle v_i^{(\ell)},x\rangle\big|\leq \mu\right\}$ defines a polygon in the subspace $\cS_\ell$, and subspace incoherence $\mu$ is given by the smallest such polytope that contains the projections of all external point $y$ into the this subspace.}\label{fig:SubspaceIncoherence}
\end{figure}


\begin{definition}[inradius]
The inradius of a convex body $\mathcal{P}$, denoted by $r(\mathcal{P})$, is defined as the radius of the largest Euclidean ball inscribed in $\mathcal{P}$.
\end{definition}
The inradius of a $\mathcal{Q}_{-i}^{(\ell)}$ describes the dispersion of the data points. Well-dispersed data lead to larger inradius and skewed/concentrated distribution of data have small inradius. An illustration is given in Figure~\ref{fig:inradius}.

\begin{definition}[Deterministic noise model]
Consider arbitrary additive noise $Z$ to $Y$, each column $z_i$ is bounded by the two quantities below:
\begin{align*}
  \delta:= \max_i\|z_i\|, && \delta_1:=\max_{i,\ell}\|\mathbb{P}_\mathcal{S_{\ell}}z_i\|,
\end{align*}
\end{definition}
As we assume the uncorrupted data point $y$ has unit norm, $\delta$ essentially describes the amount of allowable relative error.

\begin{theorem}\label{thm:thm_general}
Under the deterministic noise model, compactly denote
\begin{align*}
\mu_{\ell}:=\mu(\mathcal{X}_{\ell}),&& r_{\ell}:=\min_{\{i: x_i\in \mathcal{X}_{\ell}\}}r(\mathcal{Q}^{(\ell)}_{-i}),&&
   r:=\min_{\ell=1,...,L} r_{\ell}.
\end{align*}
If $\mu_{\ell}< r_{\ell}$ for each $\ell = 1,...,L$, furthermore
$$ \delta\leq \min_{\ell=1,...,L}\frac{r(r_{\ell}-\mu_{\ell})}{2+7r_{\ell}} $$
then LASSO subspace detection property holds for all weighting parameter $\lambda$ in the range
\begin{equation*}
\frac{1}{r - 2\delta-\delta^2}<
        \lambda<\min_{\ell=1,..,L}\left\{\frac{r_{\ell}-\mu_{\ell}-2\delta_1}{\delta(1+\delta)(2+r_{\ell}-\delta_1)}\right\}
\end{equation*}
which is guaranteed to be non-empty.
\end{theorem}
We now offer some discussions of the theorem and the proof will be given in  Section~\ref{sec:proof_deterministic}.

\paragraph{Noiseless case.} When $\delta=0$,  i.e., there is no noise, the condition reduces to $\mu_{\ell}< r_{\ell}$, which coincides with the result in \citet{soltanolkotabi2011geometric}. The exact LP formulation~\eqref{eq:SSC} is equivalent to $\lambda \rightarrow \infty$. Our result implies that unconstrained LASSO formulation \eqref{eq:Lasso} works for any $\lambda>\frac{1}{r}$. 

\paragraph{Signal-to-Noise Ratio.} Condition $\delta\leq\frac{r(r-\mu)}{2+7r}$ can be interpreted as the breaking point under increasing magnitude of attack. This suggests that SSC by \eqref{eq:Lasso} is provably robust to arbitrary noise having signal-to-noise ratio~(SNR) greater than $\Theta\big(\frac{1}{r(r-\mu)}\big)$. (Notice that $0<r<1$, and hence $7r+2 =\Theta(1)$.)

\paragraph{Tuning parameter $\lambda$.} The range of the parameter $\lambda$ in the theorem depends on unknown parameters $\mu$, $r$ and $\delta$, and therefore cannot be used in practice to choose the parameter in practice. It does however justify that when $\delta$ is small, the range of $\lambda$ that Lasso-SSC works is large, therefore not hard to tune. In practice, we do not need to know $\lambda$ in prior. One approach is to trace the Lasso path \citep{tibshirani2013lasso} until we have about $k$ non-zero entries in the coefficient vector. If we would like to use a single $\lambda$ for all columns, a good point to start is to take $\lambda$ to be in the order of $O\big(\frac{1}{\min_j \max_{i\neq j}|x_i^Tx_j|}\big)$, this ensures the solution to be at least non-trivial.

\paragraph{Agnostic subspace clustering.}
The robustness to deterministic error is important, since in practice the union-of-subspace structures are usually only good approximations. If each subspace has decaying singular values (e.g., motion segmentation, face clustering \citep{elhamifar2012ssc_journal} and hybrid system identification\citep{vidal2003algebraic}), the deterministic guarantee allows for the flexibility in choosing the cut-off points, e.g., take 90\% of the energy as signal and treat the remaining spectrum as noise. If one keeps a smaller number of singular values ( a smaller subspace dimension), the inradius will likely to be larger \footnote{A formal relationship between inradius and smallest singular value is described in \citep{wang2013provable}.}, although the noise level also increases. It is  possible that the conditions in Theorem~\ref{thm:thm_general} are satisfied for some decomposition (e.g., those with a large spectral gap) but not others. The nice thing is that this is not a tuning parameter, but rather a theoretical property that remains agnostic to the users. In fact, the algorithm will be provably effective as long as the conditions are satisfied for any signal noise decomposition (not restricted to rank-projection). None of these is possible if distributional assumptions are made to either the data or the noise.

\subsection{Randomized models}
We  further analyze three randomized models with increasing level of randomness.
\begin{description}
  \item[$\bullet$ Determinitic+Random Noise.] Subspaces and samples in subspace are arbitrary; the noise obeys the Random Noise model (Definition~\ref{def:Random_noise_model}).
  \item[$\bullet$ Semi-random+Random Noise.] Subspace is deterministic, but samples in each subspace are drawn iid uniformly from the intersection of the unit sphere and the subspace; the noise obeys the Random Noise model.
  \item[$\bullet$ Fully random.] Both subspace and samples are drawn uniformly at random from their respective domains; the noise is iid Gaussian.
\end{description}
In each of these models, we improve the performance guarantee over our conference version \citep{wang2013noisy}. In the most well-studied semi-random model, we are able to handle cases where the noise level is much larger than the signal, and hence improves upon the best known result for SSC \citet{soltanolkotabi2013robust}. A detailed comparison of the noise tolerance of these methods is given in Table~\ref{tab:comparison}.

\begin{definition}[Random noise model]\label{def:Random_noise_model}
Our random noise model is defined to be any additive $Z$ that is (1) columnwise iid; (2) spherical symmetric;  and (3) $\|z_i\|\leq \delta$ for all $i=1,...,N$ with probability at least $1-1/N$.
\end{definition}
A good example of our random noise model is iid Gaussian noise. Let each entry $Z_{ij} \sim N(0,\sigma^2/n)$. It is known that (see Lemma~\ref{lemma:random_gaussian}) for some constant $C$
$$\P\left(\delta:=\max_i\|z_i\| > \sqrt{1+\frac{6\log N}{n}}\sigma\right) \leq C/N^2.$$

\begin{theorem}[Deterministic+Random Noise]\label{thm:thm_random_noise}
 Under random noise model, compactly denote $r_{\ell}$, $r$ and $\mu_{\ell}$ as in Theorem~\ref{thm:thm_general}, furthermore let
$$\epsilon := \sqrt{\frac{6\log N}{n-\max_{\ell}{d_{\ell}}}}\leq \sqrt{\frac{C\log(N)}{n}}.$$
 If $\mu_{\ell}<r_{\ell}$ for all $\ell = 1,...,k$,
 \begin{align*}
 \epsilon\delta<\min_{\ell=1,...,L}\frac{r_{\ell}-\mu_{\ell}}{2\sqrt{d_{\ell}}+2}, &&\text{and}&& \epsilon\delta(1+\delta) < \min_{\ell=1,...,L}\frac{r(r_\ell-\mu_\ell)}{4r_\ell+6},
\end{align*}
then with probability at least $1-9/N$, LASSO subspace detection property holds for all weighting parameter $\lambda$ in the range
\begin{equation}\label{eq:thm_rand_noise_lambda_range}
\frac{1}{r- 2\epsilon \delta-\epsilon\delta^2}<
        \lambda<\min_{\ell=1,...,L}\left\{\frac{r_{\ell}-\mu_{\ell}-\delta\epsilon - \delta \sqrt{d_{\ell}} \epsilon}{\epsilon\delta(1+\delta)(3+r_{\ell}-\delta\sqrt{d_{\ell}}\epsilon)}\right\}
\end{equation}
which is guaranteed to be non-empty.
\end{theorem}
\paragraph{Low SnR paradigm.} Compared to Theorem~\ref{thm:thm_general}, Theorem~\ref{thm:thm_random_noise} considers a more benign noise which leads to a stronger result. In particular, without assuming any statistical model on how data are generated, we show that Lasso-SSC is able to tolerate noise of level $O\left((\frac{n}{\log N})^{1/4}(r(r_\ell-\mu_\ell))^{1/2}\right)$ or $O\left((\frac{n}{d\log N})^{1/2}(r_\ell-\mu_\ell)\right)$ (whichever is smaller). This extends SSC's guarantee with deterministic data to cases where the noise can be significantly larger than the signal. In fact, the SnR can go to $0$ as the ambient dimension gets large.

On the other hand, Theorem~\ref{thm:thm_random_noise} shows that Lasso-SSC is able to tolerate a constant level of noise when the geometric gap $r_\ell-\mu_\ell$ is as small as $O(\sqrt{d/n})$. This is arguably near-optimal (when $d$ is small) as the projection of a constant-level random noise into a $d$-dimensional subspace has an expected magnitude of the same order, which could easily close up the small geometric gap for some non-trivial probability if the noise is much larger.


\paragraph{Margin of error.}
Since the bound depends critically on $(r_\ell-\mu_\ell)$~--~the difference of inradius and incoherence~--~which is the geometric gap that appears in the noiseless guarantee of \citet{soltanolkotabi2011geometric}. We will henceforth call this gap the \emph{margin of error}.

We now analyze this margin of error under different generative models. We
start from the semi-random model, where the distance between two subspaces is measured as follows.
\begin{definition}\label{def:subspace_affinity}
The {\em affinity} between two subspaces is defined by:
$$ \mathrm{aff}(\mathcal{S}_k,\mathcal{S}_{\ell}) = \sqrt{\cos^2{\theta^{(1)}_{k\ell }}+...+\cos^2{\theta^{(\min(d_k,d_{\ell}))}_{k\ell}}},$$
where $\theta_{k\ell}^{(i)}$ is the $i^{th}$ canonical angle between the two subspaces. Let $U_{k}$ and $U_{\ell}$ be a set of orthonormal bases of each subspace, then $\mathrm{aff}(\mathcal{S}_k,\mathcal{S}_{\ell})=\|U_{k}^TU_{\ell}\|_F$.
\end{definition}
When data points are randomly sampled from each subspace, the geometric entity $\mu(\mathcal{X}_{\ell})$ can be expressed using this (more intuitive) subspace affinity, which leads to the following theorem.

\begin{theorem}[Semi-random model+random noise]\label{thm:semirandom}
Under the semi-random model with random noise, there exists a non-empty range of $\lambda$ such that LASSO subspace detection property holds with probability $1- \frac{9}{N} - \frac{1}{L^2}\sum_{\ell\neq \ell^\prime}\frac{1}{(N_{\ell}+1)N_{\ell^\prime}} e^{-\frac{t}{4}} -6\sum_{\ell=1}^L (e^{\gamma_1 (n-d_\ell)}+e^{\gamma_2 d_\ell}+e^{-\sqrt{N_{\ell}d_{\ell}}})$ as long as the noise level obeys
\begin{align*}
 \delta(1+\delta) \leq& \max_{\ell,\ell'}\sqrt{\frac{n-d}{6\log N}} \frac{\sqrt{\log \kappa}}{40K_2\sqrt{dd_\ell}}\left(1- \frac{K_1K_2 \mathrm{aff}(\mathcal{S}_{\ell},\mathcal{S}_{\ell^\prime})}{\sqrt{ d_{\ell^\prime}}} \right),
\end{align*}
where $K_1:= (t \log  [(N_{\ell}+1)N_{\ell^\prime}] + \log L)$, $K_2 := 4\sqrt{\frac{1}{\log\kappa_\ell}}$, $\kappa_\ell := N_{\ell}/d_{\ell}$, $\frac{\log\kappa}{d} :=\min_{\ell} \frac{\log\kappa_\ell}{d_\ell}$, and
$\gamma_1,\gamma_2$ are absolute constants.
\end{theorem}
The proof is essentially substituting the incoherence and inradius parameters in Theorem~\ref{thm:thm_random_noise} with meaningful bounds, so Thereom~\ref{thm:semirandom} can be regarded as a corollary of Theorem~\ref{thm:thm_random_noise}.

\paragraph{Overlapping subspaces.}
Similar to the results in \citet{soltanolkotabi2011geometric}, Theorem~\ref{thm:semirandom} demonstrates that LASSO-SSC can handle overlapping subspaces with noisy samples. By Definition~\ref{def:subspace_affinity}, $\mathrm{aff}(\mathcal{S}_k,\mathcal{S}_{\ell})$ can be small even if $\mathcal{S}_k$ and $\mathcal{S}_{\ell}$ share a basis.

\paragraph{Comparison to \citet{soltanolkotabi2013robust}.}
In the high dimensional setting when $n\gg d$, our result is able to handle the low SnR regime when $\delta = \Theta(n^{1/4}/d^{1/2})$, while \citet{soltanolkotabi2013robust} needs $\delta$ to be bounded by a small constant.

In the case when $d$ is a constant fraction of $n$, however, our bound is worse by a factor of $\sqrt{d}$. \citet{soltanolkotabi2013robust} is still able to handle a small constant noise while we needs $\delta < O(\frac{1}{\sqrt{d}})$. The suboptimal bound might be due to the fact that we are simply developing the theorem for the semirandom model as a corollary of Theorem~\ref{thm:thm_random_noise} and haven not fully exploit the structure of the semi-random model in the proof.


We now turn to the fully random case.
\begin{theorem}[Fully random model]\label{thm:fullrandom}
Suppose there are $L$ subspaces each with dimension $d$, chosen independently and uniformly at random. For each subspace, $\kappa d+1$ points are chosen independently and uniformly from the unit sphere inside each subspace. Each measurement is corrupted by iid Gaussian noise $\sim N(0,\sigma^2/n)$. Furthermore, if
\begin{align*}
  d < \frac{c(\kappa)^2\log\kappa}{24\log N} n, &&\text{and} && \sigma(1+\sigma) < \frac{c(\kappa)^2\log \kappa  }{20}\frac{\sqrt{n} }{d},
\end{align*}
then with probability at least $1-\frac{10}{N}-Ne^{-\sqrt{\kappa}d}$, the LASSO subspace detection property holds for any $\lambda$ in the range
\begin{equation}\label{eq:thm_rand_lambda_range}
  \frac{C_1\sqrt{d}}{c(\kappa)\sqrt{\log \kappa}}<\lambda <  \frac{C_2c(\kappa)\sqrt{n\log\kappa}}{\sigma\sqrt{d\log N}},
\end{equation}
which is guaranteed to be non-empty. Here, $C_1,C_2$ are absolute constants.
\end{theorem}
The results under this simple model are very interpretable. It
 provides intuitive guideline in how robustness of Lasso-SSC change with respect to the various parameters of the data.
 One one hand, it is sensitive to the dimension of each subspace $d$, since the $\sigma \leq \tilde{\Theta}(\frac{n^{1/4}}{\sqrt{d}})$. This dependence on subspace dimension $d$ is not a critical limitation as most interesting applications indeed have very low subspace-dimension, as summarized in Table~\ref{tab:low_rank}.
On the other hand, the dependence on the number of subspaces $L$ (in both $\log\kappa$ and $\log N$ since $N=L(\kappa d+1)$) is only logarithmic.  This suggests that SSC is robust even when there are many clusters, and $Ld\gg n$.


%

\begin{table}
  \centering
 \begin{tabular}{|p{0.6\linewidth}|c|}
   \hline
   \textbf{Application} & \textbf{Cluster rank}\\
   \hline
   3D motion segmentation \citep{costeira1998motion_seg} & $\mathrm{rank}=4$ \\\hline
   Face clustering (with shadow) \citep{basri2003lambertianface} & $\mathrm{rank}=9$ \\\hline
   Diffuse photometric face \citep{zhou2007PhotometricFace}& $\mathrm{rank}=3$ \\\hline
   Network topology discovery \citep{eriksson2011high_rankMC} & $\mathrm{rank}=2$ \\\hline
   Hand writing digits \citep{hastie1998MNIST}&  $\mathrm{rank}=12$\\\hline
   Social graph clustering \citep{xu2011graphclustering}& $\mathrm{rank}=1$ \\
   \hline
 \end{tabular}
  \caption{Rank of real subspace clustering problems}\label{tab:low_rank}
\end{table}


\subsection{Geometric interpretations}
A geometric illustration of the condition in Theorem~\ref{thm:thm_general} is given in Figure~\ref{fig:geom_interpretation} in comparison to the geometric separation condition in the noiseless case.

\begin{figure}
        \centering
        \begin{subfigure}[t]{0.4\textwidth}
          \centering
              \includegraphics[width=\textwidth]{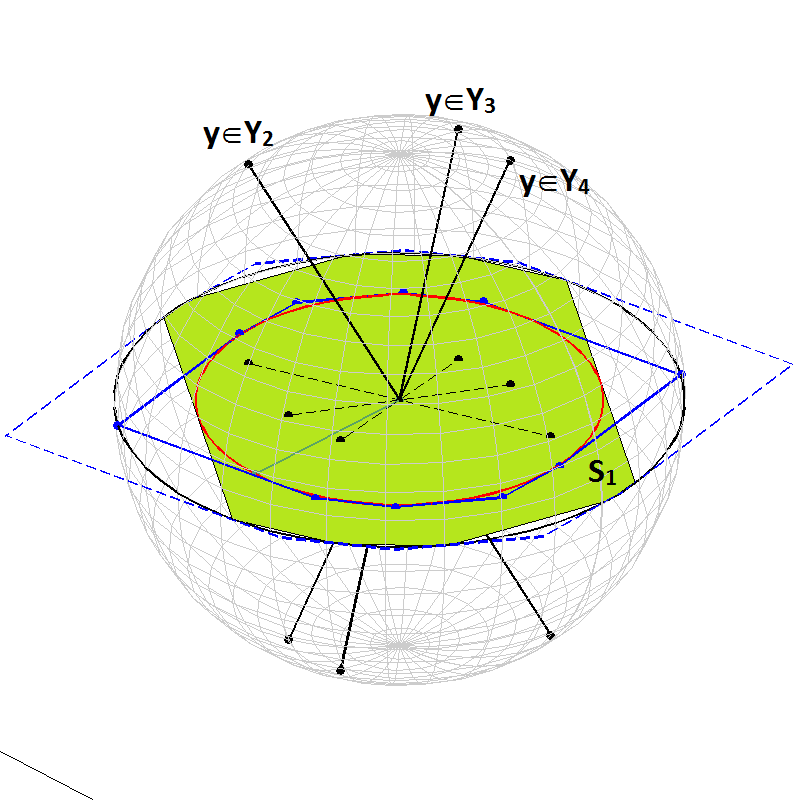}\\
               \caption{Noiseless SSC}
               \label{fig.noiseless_guarantee}
        \end{subfigure}%
        ~ 
        \begin{subfigure}[t]{0.4\textwidth}
              \centering
              \includegraphics[width=\textwidth]{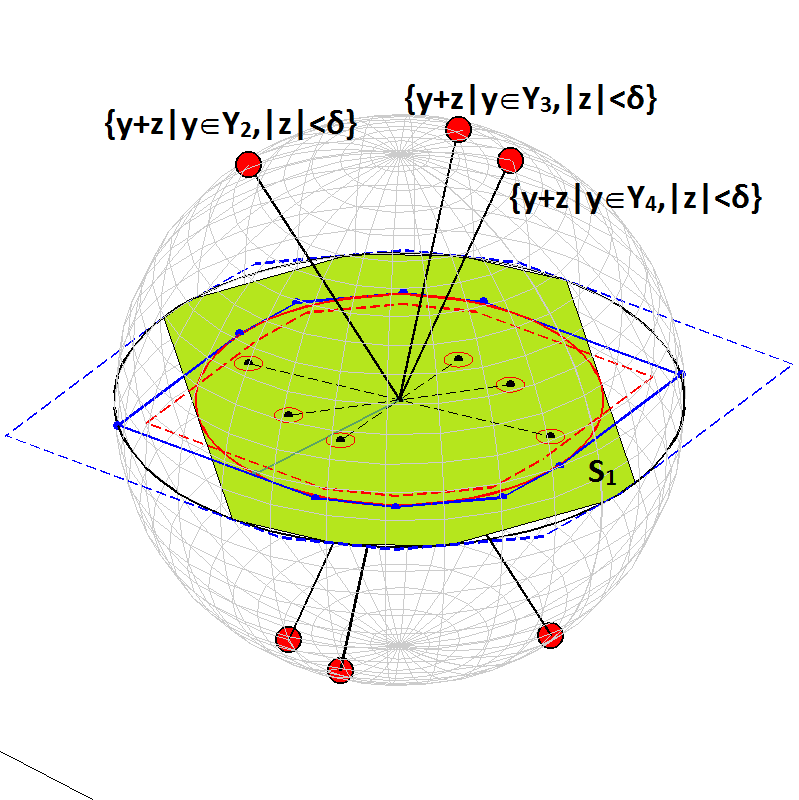}\\
              \caption{Noisy LASSO-SSC} \label{fig.noisy_guarantee}
        \end{subfigure}
   \caption{Geometric interpretation and comparison of the noiseless SSC (\textbf{Left}) and noisy LASSO-SSC (\textbf{Right}).
   }
   \label{fig:geom_interpretation}
\end{figure}

The left pane depicts the separation condition $\mu_\ell \leq r_\ell$ in Theorem~2.5 of \citet{soltanolkotabi2011geometric}. The blue polygon represents the the intersection of halfspaces defined with dual directions that are also the tangent to the red inscribing sphere. More precisely, this is $\left\{x\in \cS_\ell \middle| \big|\langle v_i^{\ell}, x\rangle\big| \leq r_\ell\right\}$. From our illustration of $\mu$ in Figure~\ref{fig:SubspaceIncoherence}, we can easily tell that $\mu_\ell\leq r_\ell$ if and only if the projection of external data points fall inside this solid blue polygon. We call this solid blue polygon the successful region.

The right pane illustrates our guarantee of Theorem~\ref{thm:thm_general} under bounded deterministic noise. The successful condition requires that the whole red ball (analogous to uncertainty set in Robust Optimization \citep{ben1998robust,bertsimas2004price}) around each external data point to fall inside the dashed red polygon, which is smaller than the blue polygon by a factor related to the noise level and the inradius.

The successful region is affected by the noise because the design matrix is also arbitrarily perturbed and the dual solution is no longer within each subspace $\cS_\ell$. Specifically, as will become clear in the proof, the key of showing SEP boils down to proving
$
\langle \nu_i^{(\ell)}, x_j\rangle < 1
$
for all pairs of $(\nu_i^{(\ell)}, x_j)$ where
$$ \nu_i^{(\ell)}=\arg\max_{\nu} \langle \nu, x_i^{(\ell)}\rangle  -\frac{1}{2\lambda}\|\nu\|^2 \text{ s.t. }  \|\nu^TX_{-i}^{(\ell)} \|_\infty \leq 1,$$
and $x_j$ is any point from another subspace. In the noiseless case we can always take $\nu_i^{(\ell)}\in\cS_\ell$ and $\langle \nu_i^{(\ell)}, x_j\rangle \leq \frac{\mu_\ell}{r_\ell}$. For noisy data and Lasso-SSC, we can no longer do that. In fact, for any fixed $\lambda$, the dual solution will be uniquely determined by a projection of $\lambda x_i^{(\ell)}$ on to the feasible region $\|\nu^TX_{-i}^{(\ell)} \|_\infty \leq 1$ (see the first pane of Figure~\ref{fig:SubspaceIncoherence}). The absolute value of the inner product $\langle \nu_i^{(\ell)}, x_j\rangle$ will depend on the magnitude of the dual solution, especially its component perpendicular to the current subspace. Indeed by carefully choosing the error, we can make $\mathbb{P}_{\cS_\ell^{\perp}}\nu$ very correlated with some external data point $x_j$.


To illustrate this further, we plot the shape of this feasible region in 3D (see Figure~\ref{fig.convex_hull_n_polar}(b)). From the feasible region alone, it seems that the magnitude of dual variable can potentially be quite large. Luckily, the quadratic penalty in the objective function allows us to exploit the optimality of the solution $\nu$ and bound the ``out-of-subspace'' component of $\nu$, which results in a much smaller region where the solution can potentially be (given in Figure~\ref{fig.convex_hull_n_polar}(c)). The region for the ``in-subspace'' component is also smaller as is shown in Figure~\ref{fig.ProjPolar}. A more detailed argument of this is given in Section~\ref{sec:dual_separation} of the proof.

\begin{figure}
  \centering
    \includegraphics[width=0.7\linewidth]{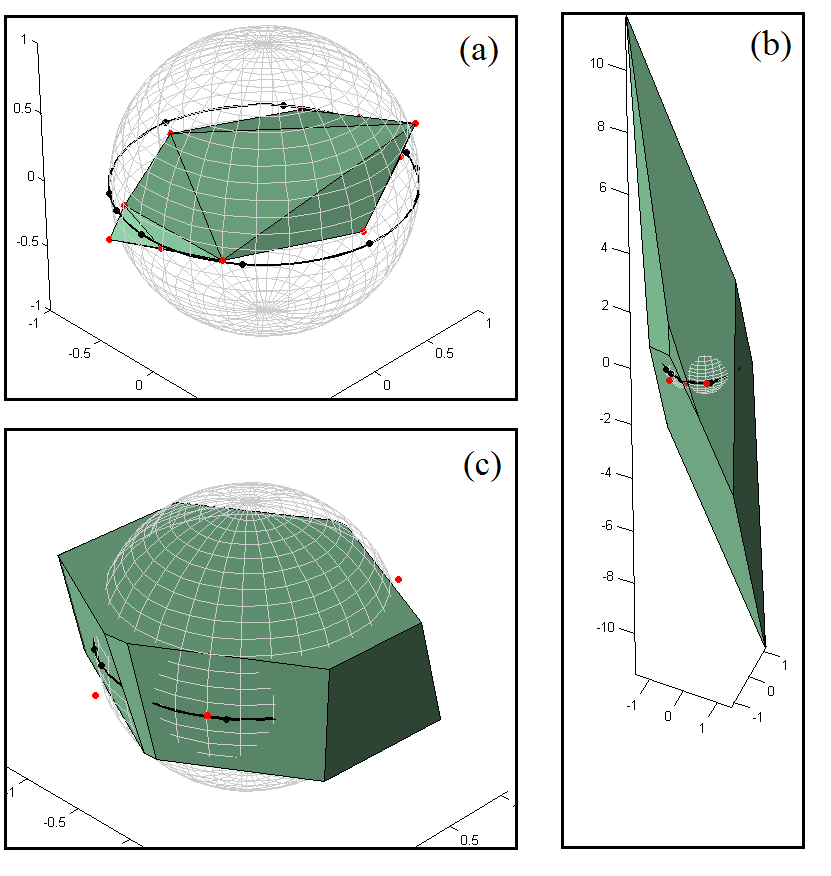}\\
  \caption{Illustration of \textbf{(a)} the convex hull of noisy data points, \textbf{(b)} its polar set and \textbf{(c)} the intersection of polar set and $\|\nu_2\|$ bound. The polar set (b) defines the feasible region of \eqref{eq:dual_fictitious2}. It is clear that $\nu_2$ can take very large value in (b) if we only consider feasibility. By considering optimality, we know the optimal $\nu$ must be inside the region in (c).} \label{fig.convex_hull_n_polar}
    \includegraphics[width=0.6\linewidth]{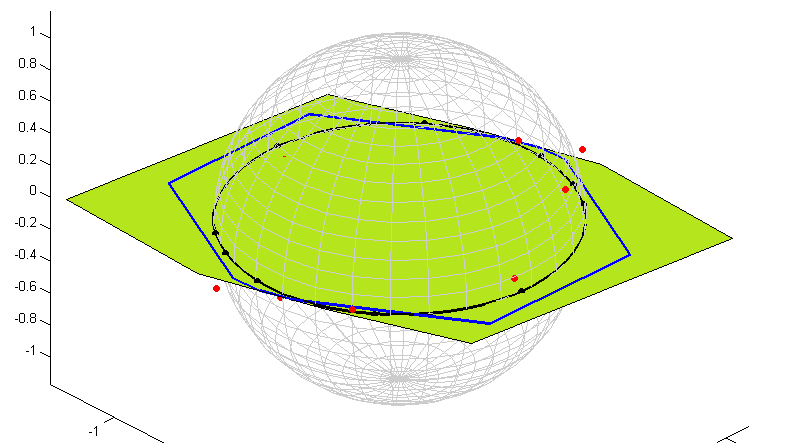}\\
  \caption{The projection of the polar set (the green area) in comparison to the projection of the polar set with $\|\nu_2\|$ bound (the blue polygon). It is clear that the latter is much smaller.}\label{fig.ProjPolar}
\end{figure}

Admittedly, the geometric interpretation under noise is slightly messier than the noiseless case, but it is clear that the largest deterministic noise Lasso-SSC can tolerate must be smaller than geometric gap $r_\ell-\mu_\ell$. Theorem~\ref{thm:thm_general} show that a sufficient condition is $\delta \leq O(r(r_\ell-\mu_\ell))$. It remains unclear whether this gap can be closed without additional assumptions.


Finally, we note that for the random noise model in Theorem~\ref{thm:thm_random_noise}, the geometric interpretation is similar, except that the impact of the noise is weakened. Thanks to the randomness and the corresponding concentration of measure, we may bound the reduction of the successful region with a much smaller value comparing to the adversarial noise case.


\section{Proof of the Deterministic Result}\label{sec:proof_deterministic}
In this section, we provide the proof for Theorem~\ref{thm:thm_general}.

Instead of analyzing \eqref{eq:Lasso} directly, we consider an equivalent constrained version by introducing slack variable $e_i$:
\begin{equation}\label{eq:Opt_original}
\begin{aligned}
\mathbf{P}_0:\quad \min_{c_i, e_i} \; \|c_i\|_1+\frac{\lambda}{2}\|e_i\|^2 \quad
s.t. \quad x^{(\ell)}_i=X_{-i}c_i+e_i.
\end{aligned}
\end{equation}
The constraint can be rewritten as
\begin{equation}\label{eq:Opt_original_equi}
y^{(\ell)}_i+z^{(\ell)}_i=(Y_{-i}+Z_{-i})c_i+e_i.
\end{equation}
The dual program of \eqref{eq:Opt_original} is:
\begin{equation}\label{eq:Opt_original_dual}
\begin{aligned}
\mathbf{D}_0:\quad \max_{\nu} \; \langle x_i,\nu \rangle - \frac{1}{2\lambda}\nu^T\nu \quad
s.t. \quad \|(X_{-i})^T\nu\|_{\infty} \leq 1.
\end{aligned}
\end{equation}
Recall that we want to establish the conditions on noise magnitude $\delta$, structure of the data ($\mu$ and $r$ in the deterministic model and affinity in the semi-random model), and ranges of valid $\lambda$ such that by Definition~\ref{def:lasso_detection}, the solution $c_i$ is \emph{non-trivial} and has support indices inside the column set $X^{(\ell)}_{-i}$ (i.e., satisfies SEP).



The proof  is hence organized into three main steps:
\begin{enumerate}
  \item[(1)] Proving SEP by duality. First we establish a set of conditions on the optimal dual variable of $D_0$ corresponding to all primal solutions satisfying SEP. Then  we construct such a dual variable $\nu$ as a certificate of proof. This is presented in Section~\ref{sec:optimality},~\ref{sec:construct_nu}~and~\ref{sec:dual_separation}.
  \item[(2)] Proving non-trivialness by showing that the optimal value is smaller than the value of the trivial solution (i.e., $c^*=0$ and $e^*=x_i^{(\ell)}$). This step is given in Section~\ref{sec:avoid_trivial}.
  \item[(3)] Showing the existence of a proper $\lambda$. As it will be made clear later, conditions for (1) include $\lambda<A$ and (2) requires $\lambda>B$ for some expression $A$ and $B$. Then it is natural to request $B<A$, so that a valid $\lambda$ exists. It turns out that this condition boils down to $\delta<C$ for some expression $C$. This argument is carried over in Section~\ref{sec:exist_lambda}.
\end{enumerate}

\subsection{Optimality Condition}\label{sec:optimality}
Consider a general convex optimization problem:
\begin{equation}\label{eq:Opt_A_general}
\begin{aligned}
\quad \min_{c, e} \; &\|c\|_1+\frac{\lambda}{2}\|e\|^2 \quad &s.t. \quad x=Ac+e.
\end{aligned}
\end{equation}
We state Lemma~\ref{lemma:OptimalCondition}, which extends Lemma~7.1 in \citet{soltanolkotabi2011geometric}.
\begin{lemma}\label{lemma:OptimalCondition}
Consider a vector $y\in \mathbb{R}^d$ and a matrix $A \in \mathbb{R}^{d\times N}$. If there exists a triplet $(c,e,\nu)$ obeying $y=Ac+e$ and $c$ has support $S\subseteq T$, furthermore the dual certificate vector $\nu$ satisfies
\begin{equation*}
\begin{array}{cccc}
  A_s^T\nu=sgn(c_S), & \nu=\lambda e,  &
  \|A^T_{T\cap S^{c}}\nu\|_{\infty} \leq 1, & \|A^T_{T^{c}}\nu\|_{\infty}<1,
\end{array}
\end{equation*}
then any optimal solution $(c^{*},e^{*})$  to \eqref{eq:Opt_A_general} obeys $c^{*}_{T^{c}}=0$.
\end{lemma}
\begin{proof}
For optimal solution $(c^{*},e^{*})$, we have:
\begin{align}
    &\|c^*\|_1+\frac{\lambda}{2}\|e^*\|^2 \nonumber\\
    =& \|c^*_S\|_1+\|c^*_{T\cap S^c}\|_1+\|c^*_{T^c}\|_1 + \frac{\lambda}{2} \|e^*\|^2\nonumber\\
    \geq&\|c_S\|_1+\langle sgn(c_S),c^*_S-c_S\rangle+\|c^*_{T\cap S^c}\|_1+\|c^*_{T^c}\|_1
    +\frac{\lambda}{2} \|e\|^2 +\langle \lambda e,e^*-e\rangle\nonumber\\
    =&\|c_S\|_1+\langle \nu,A_S(c^*_S-c_S)\rangle+\|c^*_{T\cap S^c}\|_1+\|c^*_{T^c}\|_1
    +\frac{\lambda}{2} \|e\|^2 +\langle \nu,e^*-e\rangle\nonumber\\
    =&\|c_S\|_1+\frac{\lambda}{2} \|e\|^2+ \|c^*_{T\cap S^c}\|_1-\langle \nu,A_{T\cap S^c}(c^*_{T\cap S^c})\rangle
    +\|c^*_{T^c}\|_1-\langle \nu,A_{T^c}(c^*_{T^c})\rangle. \label{eq:lemma_tmp1}
\end{align}
To see $\frac{\lambda}{2} \|e^*\|^2 \geq \frac{\lambda}{2} \|e\|^2 +\langle \lambda e,e^*-e\rangle$, note that the right hand side equals to $\lambda\left(-\frac{1}{2}e^Te +(e^*)^Te\right)$, which takes a maximal value of $\frac{\lambda}{2} \|e^*\|^2$ when $e=e^*$.
 The last equation holds because both $(c,e)$ and $(c^*,e^*)$ are feasible solution, such that $\langle\nu,A(c^*-c)\rangle+\langle\nu,e^*-e\rangle = \langle\nu,Ac^*+e^*-(Ac+e)\rangle=0$. Also, note that $\|c_S\|_1+\frac{\lambda}{2} \|e\|^2=\|c\|_1+\frac{\lambda}{2} \|e\|^2$.

With the inequality constraints of $\nu$ given in the lemma statement, we have
\begin{align*}
    \langle \nu,A_{T\cap S^c}(c^*_{T\cap S^c})\rangle=&\langle A_{T\cap S^c}^T\nu,(c^*_{T\cap S^c})\rangle
    \leq \|A^T_{T\cap S^{c}}\nu\|_{\infty}\|c^*_{T\cap S^c}\|_1\leq\|c^*_{T\cap S^c}\|_1.
\end{align*}
Substitute into \eqref{eq:lemma_tmp1}, we get:
\begin{equation*}
    \|c^*\|_1+\frac{\lambda}{2} \|e^*\|^2 \geq \|c\|_1+\frac{\lambda}{2} \|e\|^2 +(1-\|A^T_{T^{c}}\nu\|_{\infty})\|c^*_{T^c}\|_1,
\end{equation*}
where $(1-\|A^T_{T^{c}}\nu\|_{\infty})$ is strictly greater than $0$.

Using the fact that $(c^*,e^*)$ is an optimal solution, $\|c^*\|_1+\frac{\lambda}{2} \|e^*\|^2\leq \|c\|_1+\frac{\lambda}{2} \|e\|^2$. Therefore, $\|c^*_{T^c}\|_1=0$ and $(c,e)$ is also an optimal solution. This concludes the proof.
\end{proof}

The next step is to apply Lemma~\ref{lemma:OptimalCondition} with $x=x_i^{(\ell)}$ and $A=X_{-i}$ and then construct a triplet $(c,e,\nu)$ such that dual certificate $\nu$ satisfying all conditions and $c$ satisfies SEP. Then we can conclude that all optimal solutions of \eqref{eq:Opt_original} satisfy SEP.

\subsection{Construction of Dual Certificate}\label{sec:construct_nu}
To construct the dual certificate, we consider the following {\em fictitious} optimization problem (and its dual) that explicitly requires that all feasible solutions satisfy SEP\footnote{To be precise, it is the corresponding $c_i=[0,...,0,(c^{(\ell)}_i)^T,0,...,0]^T$ that satisfies SEP.} (note that one can not solve such problem in practice without knowing the subspace clusters, and hence the name ``fictitious'').
\begin{equation}\label{eq:Opt_fictitious2}
\begin{aligned}
\mathbf{P}_1: \quad \min_{c^{(\ell)}_i, e_i} \; &\|c^{(\ell)}_i\|_1+\frac{\lambda}{2}\|e_i\|^2 \quad
s.t. \quad y^{(\ell)}_i+z_i=(Y^{(\ell)}_{-i}+Z^{(\ell)}_{-i})c^{(\ell)}_i+e_i;
\end{aligned}
\end{equation}
\begin{equation}\label{eq:dual_fictitious2}
\begin{aligned}
\mathbf{D}_1: \quad \max_{\nu} \; &\langle x_i^{(\ell)},\nu \rangle - \frac{1}{2\lambda}\nu^T\nu\quad
s.t. \quad \|(X^{(\ell)}_{-i})^T\nu\|_{\infty} \leq 1.\quad\quad\quad\quad\quad\quad
\end{aligned}
\end{equation}

This optimization problem is feasible because $y^{(\ell)}_i \in span(Y^{(\ell)}_{-i})=\mathcal{S}_{\ell}$ so any $c^{(\ell)}_i$ obeying $y^{(\ell)}_i=Y^{(\ell)}_{-i}c^{(\ell)}_i$ and corresponding $e_i = z_i - Z^{(\ell)}_{-i}c^{(\ell)}_i$ is a pair of feasible solution. Then by strong duality, the dual program is also feasible, which implies that for every optimal solution $(c, e)$ of \eqref{eq:Opt_fictitious2} with $c$ supported on $S$, there exist $\nu$ satisfying:
\begin{equation*}
\left\{
\begin{aligned}
    &\|((Y^{(\ell)}_{-i})_{S^{c}}^T +(Z^{(\ell)}_{-i})_{S^{c}}^T)\nu\|_{\infty}\leq 1, \quad \nu=\lambda e, \\
    &\quad((Y^{(\ell)}_{-i})_{S}^T +(Z^{(\ell)}_{-i})_{S}^T)\nu =sgn(c_S).
\end{aligned}
\right\}
\end{equation*}
This construction of $\nu$ satisfies all conditions in Lemma~\ref{lemma:OptimalCondition} with respect to
\begin{equation}\label{eq:candidate_sol}
\begin{cases}
    c_i= [0,...,0,c_i^{(\ell)},0,...,0]\text{ with }c_i^{(\ell)}=c,\\
    e_i= e,
\end{cases}
\end{equation}
except
\begin{equation*}
    \left\|[X_1,...,X_{\ell-1},X_{\ell+1},...,X_L]^T\nu\right\|_{\infty}<1,
\end{equation*}
i.e., we must check for all data point $x \in \mathcal{X}\setminus \mathcal{X}^{\ell}$,
\begin{equation}\label{eq:dual_separation_condition}
    |\langle x, \nu \rangle|< 1.
\end{equation}
Thus, if we show that the solution of \eqref{eq:dual_fictitious2} $\nu$ also satisfies \eqref{eq:dual_separation_condition}, we can conclude that $\nu$ is a dual certificate required in Lemma~\ref{lemma:OptimalCondition}, which implies that the candidate solution \eqref{eq:candidate_sol} associated with optimal $(c,e)$ of \eqref{eq:Opt_fictitious2} is indeed the optimal solution of \eqref{eq:Opt_original} and therefore SEP holds.




\subsection{Dual separation condition}\label{sec:dual_separation}

Our strategy to show \eqref{eq:dual_separation_condition} is to provide an upper bound of $|\langle x, \nu \rangle|$ then impose the inequality on the upper bound.

First, we find it appropriate to project $\nu$ to the subspace $\mathcal{S}_{\ell}$ and its orthogonal complement subspace then analyze separately. For convenience, denote $\nu_1 := \mathbb{P}_{S_\ell}(\nu)$, $\nu_2 := \mathbb{P}_{\mathcal{S}_{\ell}^\perp}(\nu)$. Then
\begin{equation}\label{eq:showing_dual_sep_cond}
\begin{aligned}
    |\langle x, \nu \rangle| =& |\langle y+z, \nu\rangle|\leq |\langle y, \nu_1\rangle|+|\langle y,\nu_2\rangle|+|\langle z, \nu\rangle|\\
    \leq& \mu(\mathcal{X}_{\ell}) \|\nu_1\| + \|y\|\|\nu_2\||\cos(\angle (y,\nu_2))|
     + \|z\|\|\nu\||\cos(\angle (z,\nu))|.
\end{aligned}
\end{equation}
To see the last inequality, check that by Definition~\ref{def:incoherence}, $|\langle y,\frac{\nu_1}{\|\nu_1\|}\rangle| \leq\mu(\mathcal{X}_{\ell})$.

Since we are considering general (possibly adversarial) noise, we will use the relaxation $|\cos(\theta)|\leq 1$ for all cosine terms (a better bound under random noise will be given later). Thus, what left is to bound $\|\nu_1\|$ and $\|\nu_2\|$ (note $\|\nu\|=\sqrt{\|\nu_1\|^2+\|\nu_2\|^2} \leq \|\nu_1\|+\|\nu_2\|$).

\subsubsection{Bounding \texorpdfstring{$\|\nu_1\|$}{||nu1||}}
We first bound $\|\nu_1\|$ by exploiting the feasible region of $\nu_1$ in \eqref{eq:dual_fictitious2}:
$$\left\{\nu \middle| \|(X^{(\ell)}_{-i})^T\nu\|_{\infty} \leq 1\right\},$$
which is equivalent to
$$\left\{\nu \middle| x_j^T\nu\leq 1 \quad\text{for every column $x_j$ of $X^{(\ell)}_{-i}$}\right\}.$$
Decompose the condition into
$$y_j^T\nu_1+(\mathbb{P}_{\mathcal{S}_{\ell}} z_j)^T\nu_1+ z_j^T\nu_2\leq 1.$$
and relax the expression into
\begin{equation}\label{eq:relax_constraint}
  y_j^T\nu_1+(\mathbb{P}_{\mathcal{S}_{\ell}} z_j)^T\nu_1 \leq 1-z_j^T\nu_2\leq 1+\delta\|\nu_2\|.
\end{equation}
The relaxed condition contains the feasible region of $\nu_1$ in \eqref{eq:dual_fictitious2}.
It turns out that the geometric properties of this relaxed feasible region provides an upper bound of $\|\nu_1\|$.
\begin{definition}[polar set]
The polar set $\mathcal{K}^o$ of set $\mathcal{K} \in \mathbb{R}^d$ is defined as
\begin{equation*}
    \mathcal{K}^o = \left\{y\in \mathbb{R}^d: \langle x,y\rangle \leq 1\text{ for all } x\in \mathcal{K}\right\}.
\end{equation*}
\end{definition}
By the polytope geometry, we have
\begin{equation}\label{eq:Geometric_dual}
\begin{aligned}
  \|(Y_{-i}^{(\ell)}+\mathbb{P}_{\mathcal{S}_{\ell}}(Z_{-i}^{(\ell)}))^T\nu_1\|_{\infty} \leq  1+\delta\|\nu_2\|
  \;\Leftrightarrow \; \nu_1 \in \left[\mathcal{P}\left(\frac{Y_{-i}^{(\ell)}+\mathbb{P}_{\mathcal{S}_{\ell}}(Z_{-i}^{(\ell)})}{1+\delta\|\nu_2\|}\right)\right]^o := \mathcal{T}^o.
\end{aligned}
\end{equation}
Now we introduce the concept of circumradius.
\begin{definition}[circumradius]
The circumradius of a convex body $\mathcal{P}$, denoted by $R(\mathcal{P})$, is defined as the radius of the smallest Euclidean ball containing $\mathcal{P}$.
\end{definition}
The magnitude $\|\nu_1\|$ is bounded by $R(\mathcal{T}^o )$. Moreover, by the the following lemma we may find the circumradius by analyzing the polar set of $\mathcal{T}^o$ instead. By the property of polar operator, polar of a polar set gives the tightest convex envelope of the original set, i.e., $(\mathcal{K}^o)^o = conv(\mathcal{K})$. Since $\mathcal{T}=\mathrm{conv}\left(\pm \frac{Y_{-i}^{(\ell)}+\mathbb{P}_{\mathcal{S}_{\ell}}(Z_{-i}^{(\ell)})}{1+\delta\|\nu_2\|}\right)$ is convex in the first place, the polar set of $\mathcal{T}^o$ is  $\mathcal{T}$.
\begin{lemma}[Page 448 in \citet{brandenberg2004isoradial}]\label{lemma:circum_inradius}
For a symmetric convex body $\mathcal{P}$, i.e. $\mathcal{P}=-\mathcal{P}$, inradius of $\mathcal{P}$ and circumradius of polar set of $\mathcal{P}$ satisfy:
\begin{equation*}
    r(\mathcal{P})R(\mathcal{P}^o)=1.
\end{equation*}
\end{lemma}
\begin{lemma}\label{lemma:Y_containing_set}
Given $X=Y+Z$,  denote $\rho:=\max_{i}\|\mathbb{P}_\mathcal{S}z_i\|$, furthermore $Y\in \mathcal{S}$ where $\mathcal{S}$ is a linear subspace, then we have:
\begin{equation*}
    r(\mathrm{Proj}_\mathcal{S} (\mathcal{P}(X))) \geq r(\mathcal{P}(Y)) - \rho
\end{equation*}
\end{lemma}
\begin{proof}
First note that projection to a subspace is a linear operator. Hence $\mathrm{Proj}_\mathcal{S}(\mathcal{P}(X))=\mathcal{P}(\mathbb{P}_\mathcal{S} X)$. Then by definition, the boundary set of $\mathcal{P}(\mathbb{P}_\mathcal{S} X)$ is $\mathcal{B}:=\left\{y\text{ }|\text{ }y=\mathbb{P}_\mathcal{S} X c; \|c\|_1=1\right\}$. Inradius by definition is the largest ball containing in the convex body, hence $r(\mathcal{P}(\mathbb{P}_\mathcal{S} X)) = \min_{y\in \mathcal{B}} \|y\|$. Now we provide a lower bound of it:
\begin{align*}
\|y\| \geq& \|Yc\|-\|\mathbb{P}_\mathcal{S}Z c\|\geq r(\mathcal{P}(Y)) - {\sum}_j{\|\mathbb{P}_\mathcal{S}z_j}\||c_j|
\geq r(\mathcal{P}(Y)) - \rho\|c\|_1.
\end{align*}
This concludes the proof.
\end{proof}
A bound of $\|\nu_1\|$ follows directly from Lemma~\ref{lemma:circum_inradius} and Lemma~\ref{lemma:Y_containing_set}:
\begin{align}
\|\nu_1\| \leq& (1+\delta\|\nu_2\|)R(\mathcal{P}(Y_{-i}^{(\ell)}+\mathbb{P}_{\mathcal{S}_{\ell}}(Z_{-i}^{(\ell)})))\nonumber\\
=& \frac{1+\delta\|\nu_2\|}{r(\mathcal{P}(Y_{-i}^{(\ell)}+\mathbb{P}_{\mathcal{S}_{\ell}}(Z_{-i}^{(\ell)}))}
=\frac{1+\delta\|\nu_2\|}{r(\mathrm{Proj}_\mathcal{S_{\ell}} (\mathcal{P}(X_{-i}^{(\ell)})))}
\leq \frac{1+\delta\|\nu_2\|}{r{\left( \mathcal{Q}_{-i}^{\ell}\right)}-\delta_1}.\label{eq:nu1_bound}
\end{align}
This bound depends on $\|\nu_2\|$, which we analyze below.

\subsubsection{Bounding \texorpdfstring{$\|\nu_2\|$}{||nu2||}}

Since $\nu$ is the optimal solution to $\mathbf{D}_1$, it obeys the second optimality condition in Lemma~\ref{lemma:OptimalCondition}:
$$ \nu=\lambda e_i=\lambda(x_i-X^{(\ell)}_{-i}c). $$
By projecting $\nu$ to $\mathcal{S}^{\perp}_{\ell}$, we get
$\nu_2=\lambda \mathbb{P}_{\mathcal{S}_{\ell}^{\perp}}(x_i-X^{(\ell)}_{-i}c) = \lambda \mathbb{P}_{\mathcal{S}_{\ell}^{\perp}}(z_i-Z^{(\ell)}_{-i}c)$.
It follows that
\begin{align}
  \|\nu_2\|&\leq\lambda \left(\|\mathbb{P}_{\mathcal{S}_{\ell}^{\perp}}z_i\|+\|\mathbb{P}_{\mathcal{S}_{\ell}^{\perp}}Z^{(\ell)}_{-i}c\|\right)\nonumber\\
  & \leq \lambda\left( \|\mathbb{P}_{\mathcal{S}_{\ell}^{\perp}}z_i\| + \sum_{j} |c_j|\|\mathbb{P}_{\mathcal{S}_{\ell}^{\perp}}z_j\|\right) \nonumber\\
  & \leq \lambda (\|c\|_1+1)\delta_2 \leq \lambda (\|c\|_1+1)\delta. \label{eq:bounding_nu2}
\end{align}

Now we bound $\|c\|_1$.
Since $(c,e)$ is the optimal solution, $\|c\|_1+\frac{\lambda}{2}\|e\|^2 \leq \|\tilde{c}\|_1 + \frac{\lambda}{2}\|\tilde{e}\|^2$ for any feasible solution $(\tilde{c},\tilde{e})$. Let $\tilde{c}$ be the solution of
\begin{equation}\label{eq:Opt_y only}
\begin{aligned}
\min_{c} \; \|c\|_1 \quad
s.t. \quad y^{(\ell)}_i=Y^{(\ell)}_{-i}c,
\end{aligned}
\end{equation}
then by strong duality,
$$\|\tilde{c}\|_1 = \max_{\nu}\left\{\langle\nu,y^{(\ell)}_i\rangle \text{ }|\text{ } \|[Y^{(\ell)}_{-i}]^T\nu\|_{\infty}\leq 1\right\}.$$
By Lemma~\ref{lemma:circum_inradius}, the optimal dual solution $\tilde{\nu}$ satisfies $\|\tilde{\nu}\|\leq \frac{1}{r(\mathcal{Q}_{-i}^{\ell})}$. It follows that
$$\|\tilde{c}\|_1= \langle\tilde{\nu},y^{(\ell)}_i\rangle = \|\tilde{\nu}\|\|y^{(\ell)}_i\|\leq \frac{1}{r(\mathcal{Q}_{-i}^{\ell})}.$$

On the other hand, $\tilde{e} = z_i - Z_{-i}^{(\ell)}\tilde{c}$, so $\|\tilde{e}\|^2 \leq (\|z_i\|+\sum_j \|z_j\||\tilde{c}_j|)^2\leq (\delta + \|\tilde{c}\|_1\delta)^2$, thus
$$\|c\|_1 \leq \|\tilde{c}\|_1 + \frac{\lambda}{2}\|\tilde{e}\|^2 - \frac{\lambda}{2}\|e\|^2\leq \frac{1}{r{(\mathcal{Q}_{-i}^{\ell})}}+\frac{\lambda}{2}\delta^2\left[1+\frac{1}{r{(\mathcal{Q}_{-i}^{\ell})}}\right]^2-\frac{1}{2\lambda}\|\nu_2\|^2.$$
Note that we used the property $\frac{\lambda}{2}\|e\|^2=\frac{1}{2\lambda}\|\nu\|^2\geq\frac{1}{2\lambda}\|\nu_2\|^2$. Substitute the bound of $\|c_1\|_1$ into \eqref{eq:bounding_nu2} we get
\begin{align*}
  &&\;&\|\nu_2\|\leq \lambda \left(\frac{1}{r(\mathcal{Q}_{-i}^{\ell})}+\frac{\lambda}{2}\delta^2\left[1+\frac{1}{r(\mathcal{Q}_{-i}^{\ell})}\right]^2+1\right) \delta-\frac{\delta}{2}\|\nu_2\|^2\\
  \Leftrightarrow&&\;
   &\|\nu_2\|+\frac{\delta}{2}\|\nu_2\|^2\leq \lambda\delta\left(\frac{1}{r(\mathcal{Q}_{-i}^{\ell})}+1\right)+ \frac{\delta}{2} \left[\lambda\delta\left(\frac{1}{r(\mathcal{Q}_{-i}^{\ell})}+1\right)\right]^2 .
\end{align*}
Since function $f(\alpha)=\alpha+\frac{\delta}{2}\alpha^2$ monotonically increases when $\alpha>0$, the above inequality implies
\begin{equation}\label{eq:nu2_bound}
  \|\nu_2\| \leq \lambda\delta\left(\frac{1}{r(\mathcal{Q}_{-i}^{\ell})}+1\right),
\end{equation}
which gives the desired bound for $\|\nu_2\|$.

\subsubsection{Conditions for  \texorpdfstring{$|\langle x, \nu \rangle|<1$}{|<x,nu>|}}
Putting together \eqref{eq:showing_dual_sep_cond}, \eqref{eq:nu1_bound} and \eqref{eq:nu2_bound}, we have the upper bound of $|\langle x, \nu \rangle|$:
\begin{equation*}
\begin{aligned}
    &|\langle x, \nu \rangle| \leq (\mu(\mathcal{X}_{\ell})+\|\mathbb{P}_{\mathcal{S}_{\ell}}z\|) \|\nu_1\| + (\|y\|+\|\mathbb{P}_{\mathcal{S}_{\ell}^{\perp}}z\|)\|\nu_2\|\\
    \leq&\frac{\mu(\mathcal{X}_{\ell})+\delta_1}{r{\left( \mathcal{Q}_{-i}^{\ell}\right)}-\delta_1} + \left(\frac{(\mu(\mathcal{X}_{\ell})+\delta_1)\delta}{r{\left( \mathcal{Q}_{-i}^{\ell}\right)}-\delta_1}+1+\delta\right)\|\nu_2\|\\
    \leq& \frac{\mu(\mathcal{X}_{\ell})+\delta_1}{r{\left( \mathcal{Q}_{-i}^{\ell}\right)}-\delta_1} + \lambda\delta(1+\delta) \left(\frac{1}{r(\mathcal{Q}_{-i}^{\ell})}+1\right)
    + \frac{\lambda\delta^2(\mu(\mathcal{X}_{\ell})+\delta_1)}{r{\left( \mathcal{Q}_{-i}^{\ell}\right)}-\delta_1}\left(\frac{1}{r(\mathcal{Q}_{-i}^{\ell})}+1\right).
\end{aligned}
\end{equation*}
For convenience, we further relax the second $r(\mathcal{Q}_{-i}^{\ell})$ into $r(\mathcal{Q}_{-i}^{\ell})-\delta_1$. The dual separation condition is thus guaranteed with
\begin{align*}
    \frac{\mu(\mathcal{X}_{\ell})+\delta_1 +\lambda\delta(1+\delta)+\lambda\delta^2(\mu(\mathcal{X}_{\ell})+\delta_1)}{r{\left( \mathcal{Q}_{-i}^{\ell}\right)}-\delta_1}
    + \lambda\delta(1+\delta)+\frac{\lambda\delta^2(\mu(\mathcal{X}_{\ell})+\delta_1)}{r{\left( \mathcal{Q}_{-i}^{\ell}\right)}(r{\left( \mathcal{Q}_{-i}^{\ell}\right)}-\delta_1)}  < 1.
\end{align*}
Denote $\rho:=\lambda\delta(1+\delta)$, assume $\delta<r{\left( \mathcal{Q}_{-i}^{\ell}\right)}$, $(\mu(\mathcal{X}_{\ell})+\delta_1)<1$ and simplify the form with
\begin{align*}
\frac{\lambda\delta^2(\mu(\mathcal{X}_{\ell})+\delta_1)}{r{\left( \mathcal{Q}_{-i}^{\ell}\right)}-\delta_1}+\frac{\lambda\delta^2(\mu(\mathcal{X}_{\ell})+\delta_1)}{r{\left( \mathcal{Q}_{-i}^{\ell}\right)}(r{\left( \mathcal{Q}_{-i}^{\ell}\right)}-\delta_1)}
< \frac{\rho}{r{\left( \mathcal{Q}_{-i}^{\ell}\right)}-\delta_1},
\end{align*}
we get a sufficient condition
\begin{equation}\label{eq:dual_cond_simp}
    \mu(\mathcal{X}_{\ell}) +2\rho +\delta_1< \left(1-\rho\right)(r(\mathcal{Q}_{-i}^{\ell})-\delta_1).
\end{equation}
To generalize \eqref{eq:dual_cond_simp} to all data of all subspaces, the following must hold for each $\ell = 1,...,k$:
\begin{equation}\label{eq:Thm1_all}
    \mu(\mathcal{X}_{\ell}) +2\rho+\delta_1 < \left(1-\rho\right)\left(\min_{\{i: x_i\in X^{(\ell)}\}}r(\mathcal{Q}^{(\ell)}_{-i})-\delta_1\right).
\end{equation}
This gives a first condition on $\delta$ and $\lambda$ (wihtin $\rho$), which we call it ``\textbf{dual separation condition}'' under noise. Note that this reduces to exactly the geometric condition in \citet{soltanolkotabi2011geometric}'s Theorem~2.5 when $\delta=0$.

\subsection{Avoid trivial solutions}\label{sec:avoid_trivial}
Besides SEP, we also need to show the solution is non-trivial. The idea is that when $\lambda$ is large enough, the trivial solution $c^* = 0$, $e^*=x_i^{(\ell)}$ can never be optimal.

As we trace along the regularization path by increasing $\lambda$ from $0$, one column of the design matrix $X_{-i}$ will enter the support set. This column will be the one that attains $\|X_{-i}^Tx_i\|_{\infty}$, and $\lambda = \frac{1}{\|X_{-i}^Tx_i\|_{\infty}}$ when it happens. Therefore, as long as $\lambda> \frac{1}{\|X_{-i}^Tx_i\|_{\infty}}$, the solution will not be trivial.

Note that under the dual separation condition, we only need to consider points in the same subspace. So $\|X_{-i}^Tx_i\|_{\infty} = \left\|[X_{-i}^{(\ell)}]^Tx_i\right\|_{\infty}$. Let $x_j\in X_{-i}^{(\ell)}$ be the column that attains the maximum in $\left\|[X_{-i}^{(\ell)}]^Tx_i\right\|_{\infty}$ and $y_k \in Y_{-i}^{(\ell)}$ be the column that attains the maximum in $\left\|[Y_{-i}^{(\ell)}]^Ty_i\right\|_{\infty}$ (if there are more than one maximizers, pick any one), we can write
\begin{align}
\left\|[X_{-i}^{(\ell)}]^Tx_i\right\|_{\infty}&=\left|\langle x_j, x_i\rangle\right|  \geq \left|\langle x_k, x_i\rangle\right|\nonumber\\
 &= \left|\langle y_k, y_i\rangle  + \langle y_k, z_i\rangle + \langle z_k, y_i\rangle + \langle z_k, z_i\rangle\right|\nonumber\\
&\geq \left|\langle y_k, y_i\rangle\right| - \left|\langle y_k, z_i\rangle + \langle z_k, y_i\rangle + \langle z_k, z_i\rangle\right|\nonumber\\
& = \left\|[Y_{-i}^{(\ell)}]^Ty_i\right\|_{\infty} - \left|\langle y_k, z_i\rangle + \langle z_k, y_i\rangle + \langle z_k, z_i\rangle\right|\nonumber\\
&\geq r(\mathcal{Q}^{(\ell)}_{-i}) - 2\delta - \delta^2. \label{eq:lowerbounding_max_affinity}
\end{align}
The last inequality follows from the upper bound of noise magnitude and the observation that the inradius of $\mathcal{Q}^{(\ell)}_{-i}$ defines a uniform lower bound of $\left\| [ Y_{-i}^{(\ell)}]^T w\right\|_\infty$ for any unit vector $w\in \mathcal{S}_{\ell}$. Therefore, as long as
\begin{equation}\label{eq:lambda_low}
\lambda \geq \frac{1}{ r(\mathcal{Q}^{(\ell)}_{-i}) - 2\delta - \delta^2},
\end{equation}
the solution $c_i$ for $i$th column is not trivial. This bound is strictly better than what we obtain in the conference version \citep{wang2013noisy} and is the key for improving the rate for noise tolerance over the previous version. Also, check that
\begin{equation}\label{eq:deterministic_delta}
\delta<\frac{r(r_\ell - \mu_\ell)}{2+7r_\ell}
\end{equation}
under bound of $\delta$ in the theorem statement,  $r(\mathcal{Q}^{(\ell)}_{-i}) - 2\delta - \delta^2>0$ for any $i,\ell$.

A side remark is that the Lasso regularization path is formally described in \citet{tibshirani2013lasso} and it is unique whenever the data points are in general position. As a result, we can potentially calculate the entry point of $k$th non-zero coefficient for any $0<k<d$, any $x_i$ and $X_{-i}$. This would however complicate the results unnecessarily, as Lasso path is not monotone (some coefficient may leave the support set as $\lambda$ increases). We therefore stick to the simpler requirement of $c_i$ being non-trivial.

\subsection{Existence of a proper  \texorpdfstring{$\lambda$}{lambda}}\label{sec:exist_lambda}

Basically, \eqref{eq:Thm1_all} and \eqref{eq:lambda_low} must be satisfied simultaneously for all $\ell=1,...,L$. Essentially \eqref{eq:lambda_low} gives a condition of $\lambda$ from below, and \eqref{eq:Thm1_all} gives a condition from above. Recall that the denotations $r_{\ell}:=\min_{\{i: x_i\in X^{(\ell)}\}}r(\mathcal{Q}^{(\ell)}_{-i})$, $\mu_{\ell}:=\mu(\mathcal{X}_{\ell})$ and $r=\min_{\ell} r_{\ell}$, the condition on $\lambda$ is:
\begin{align*}
  \max_{\ell}\frac{1}{r_{\ell} - 2\delta-\delta^2}<\lambda < \min_{\ell}\frac{r_{\ell}-\mu_{\ell}-2\delta_1}{\delta(1+\delta)(2+r_{\ell}-\delta_1)}.
\end{align*}
With the observation that
\begin{align*}
\max_{\ell}\frac{1}{r_{\ell} - 2\delta-\delta^2}
= \frac{1}{\min r_{\ell} - 2\delta-\delta^2},
\end{align*}
it suffices to require $\lambda$ to obey for each $\ell$:
\begin{equation}\label{eq:lambda_range}
\frac{1}{r - 2\delta-\delta^2}<
        \lambda<\frac{r_{\ell}-\mu_{\ell}-2\delta_1}{\delta(1+\delta)(2+r_{\ell}-\delta_1)}.
\end{equation}



We will now show that under condition \eqref{eq:deterministic_delta}, the range \eqref{eq:lambda_range} is not an empty set. Again, we relax $\delta_1$ to $\delta$ in \eqref{eq:lambda_range} and get
\begin{equation}\label{eq:nonempty_cond}
  \frac{1}{r - 2\delta-\delta^2}< \frac{r_{\ell}-\mu_{\ell}-2\delta}{\delta(1+\delta)(2+r_{\ell}-\delta)}.
\end{equation}
Since all denominators are positive, we obtain the standard form of the inequality
$$ A\delta^3+B\delta^2+C\delta+D<0 $$ with
$$
\begin{cases}
A=-3\leq 0\\
B=-3+2(r_\ell-\mu_\ell) + r_\ell \leq 0\\
C=2+4(r_{\ell}-\mu_{\ell})+r_\ell+2r \leq 2+7r_{\ell}\\
D=-r(r_{\ell}-\mu_{\ell})
\end{cases}
$$
Check that \eqref{eq:deterministic_delta} is sufficient for the above $3$rd order inequality to hold. Therefore,
$$\eqref{eq:deterministic_delta}\Rightarrow A\delta^3+B\delta^2+C\delta+D<0 \Leftrightarrow \eqref{eq:nonempty_cond}
\Rightarrow \text{\eqref{eq:lambda_range} is not an empty set.}$$
This completes the proof of Theorem~\ref{thm:thm_general}.

\section{Proof of  Results for Randomized Cases}\label{sec:proof_randomized}
In this section, we provide proofs to the theorems of the three randomized models:
\begin{description}
  \item[$\bullet$ Deterministic data+random noise;]
  \item[$\bullet$ Semi-random data+random noise;]
  \item[$\bullet$ Fully random.]
\end{description}

To do this, we need to bound $\delta_1$, $\cos(\angle(z,\nu))$ and $\cos(\angle(y,\nu_2))$ when $Z$ follows the \emph{Random Noise Model}, such that a better dual separation condition can be obtained. Moreover, for the \emph{Semi-random} and the \emph{Random data model}, we need to bound $r(\mathcal{Q}^{(\ell)}_{-i})$ when data samples from each subspace are drawn uniformly and bound $\mu(\mathcal{X}_{\ell})$  when subspaces are randomly generated.
These require the following lemmas.

\begin{lemma}[Upper bound on the area of spherical cap]\label{lemma:spherical_cap}
  Let $a \in \mathbb{R}^{n}$ be a random vector sampled from a unit sphere and $z$ is a fixed vector. Then we have:
  \begin{equation*}
    Pr\left(|a^Tz|>\epsilon \|z\|\right) \leq 2e^{\frac{-n\epsilon^2}{2}}
  \end{equation*}
\end{lemma}
This Lemma is extracted from an equation in page~29 of \citet{soltanolkotabi2011geometric}, which is in turn adapted from the upper bound on the area of spherical cap in \citet{ball1997intro_convex_geometry}.
By definition of the Random Noise Model, $z_i$ is spherical symmetric, which implies that the direction of $z_i$ is distributed uniformly on the $n$-dimensional unit sphere. Hence Lemma~\ref{lemma:spherical_cap} applies whenever an inner product involves $z$.
As an example, we write the following lemma.
\begin{lemma}[Properties of Gaussian noise]\label{lemma:random_gaussian}
  For Gaussian random matrix $Z\in \mathbb{R}^{n\times N}$, if each entry $Z_{i,j} \sim N(0,\frac{\sigma}{\sqrt{n}})$, then each column $z_i$ satisfies:
  \begin{align*}
      \text{1. }&Pr(\|z_i\|^2 > (1+t)\sigma^2) \leq e^{\frac{n}{2}(\log(t+1)-t)}\\
      \text{2. }&Pr(|\langle z_i,z \rangle|>\epsilon\|z_i\|\|z\|) \leq 2e^{\frac{-n\epsilon^2}{2}}
  \end{align*}
  where $z$ is any fixed vector, or a random vector that is independent to $z_i$.
\end{lemma}
\begin{proof}
The second property follows directly from Lemma~\ref{lemma:spherical_cap} as Gaussian vector has a uniformly random direction.

To show the first property, we observe that the sum of $n$ independent square Gaussian random variables follows $\chi^2$ distribution with degree of freedom $n$. In other words, we have
$$\|z_i\|^2 = |Z_{1i}|^2+...+|Z_{ni}|^2 \sim \frac{\sigma^2}{n}\chi^2(n).$$
By Hoeffding's inequality, we have an approximation of its CDF~\citep{dasgupta2002xi_square_concentration}, which gives us
$$Pr(\|z_i\|^2>\alpha\sigma^2)=1-\mathrm{CDF}_{\chi^2_n}(\alpha)\leq (\alpha e^{1-\alpha})^{\frac{n}{2}}.$$
Substitute $\alpha=1+t$, we obtain the concentration statement in the lemma.
\end{proof}

By Lemma~\ref{lemma:random_gaussian}, $\delta = \max_i \|z_i\|$ is bounded with high probability. $\delta_1$ can be bounded even more tightly because each $\mathcal{S}_{\ell}$ is low-rank. Likewise, $\cos(\angle(z,\nu))$ is bounded by a small value with high probability. Moreover, since $\nu=\lambda e=\lambda(x_i-X_{-i}c)$, $\nu_2 = \lambda\mathbb{P}_{\mathcal{S}^{\perp}_{\ell}}(z_i - Z_{-i}c)$. Thus $\nu_2$ is indeed a weighted sum of random noise in a $(n-d_{\ell})$-dimensional subspace. Consider $y$ a fixed vector, $\cos(\angle(y,\nu_2))$ is also bounded with high probability.

Replace these observations into \eqref{eq:dual_separation_condition} and the corresponding bound of $\|\nu_1\|$ and $\|\nu_2\|$, we obtain the equivalent \emph{dual separation condition} under the random noise model (equivalent to \eqref{eq:dual_cond_simp} in the proof of the deterministic case). This is formalized in the following lemma.

\begin{lemma}[Dual separation condition under random noise]\label{lemma:dual_sep_random}
Let $\rho:=\lambda\delta(1+\delta)$ and
$$\epsilon := \sqrt{\frac{6\log N}{n-\max_{\ell}{d_{\ell}}}}\leq \sqrt{\frac{C\log(N)}{n}}$$ for some constant $C$. Under random noise model, if for each $\ell=1,...,L$
\begin{equation}\label{eq:dual_sep_random}
  \mu(\mathcal{X}_{\ell})+\delta\epsilon +3\rho\epsilon  \leq (1-\rho\epsilon)(\max_{i}r(\mathcal{Q}_{-i}^{(\ell)})-\delta\sqrt{d_{\ell}}\epsilon),
\end{equation}
then dual separation condition \eqref{eq:dual_separation_condition} holds for all data points with probability at least $1-8/N$.
\end{lemma}
\begin{proof}
Recall that we want to find an upper bound of $|\langle x, \nu \rangle|$.
\begin{equation}\label{eq:showing_dual_sep_cond_rand}
\begin{aligned}
    |\langle x, \nu \rangle| \leq& \mu \|\nu_1\|+ \|y\|\|\nu_2\||\cos(\angle (y,\nu_2))|
    + \|z\|\|\nu\||\cos(\angle (z,\nu))|
\end{aligned}
\end{equation}
Here we will bound the two cosine terms and $\delta_1$ under the random noise model.

  As discussed above, directions of $z$ and $\nu_2$ are independently and uniformly distributed on the $n$-dimension unit sphere. Then by Lemma~\ref{lemma:spherical_cap},
  $$\begin{cases}
    Pr\left(\cos(\angle(z,\nu)) > \sqrt{\frac{6\log N}{n}}\right) \leq \frac{2}{N^3};\\
    Pr\left(\cos(\angle(y,\nu_2)) > \sqrt{\frac{6\log N}{n-d_{\ell}}}\right) \leq \frac{2}{N^3};\\
    Pr\left(\cos(\angle(z,\nu_2)) > \sqrt{\frac{6\log N}{n}}\right) \leq \frac{2}{N^3}.
  \end{cases}$$
  Using the same technique, we derive a bound for $\delta_1$. Given an orthonormal basis $U$ of $S_{\ell}$, $\mathbb{P}_{S_{\ell}}z = UU^Tz$, then $$\|UU^Tz\|=\|U^Tz\| =\sqrt{ \sum_{i=1,...,d_{\ell}}|U_{:,i}^Tz|^2}.$$
    Apply Lemma~\ref{lemma:spherical_cap} for each $i$,
     then by union bound, we get:
  $$Pr\left(\|\mathbb{P}_{S_{\ell}}z\| > \sqrt{\frac{ 6d_{\ell}\log N}{n}}\delta\right) \leq \frac{2d_{\ell}}{N^3}.$$
  Since $\delta_1$ is the worse case bound for all $L$ subspace and all $N$ noise vector, then a union bound gives:
  $$Pr\left(\delta_1 > \sqrt{\frac{ 6d_{\ell}\log N}{n}}\delta\right) \leq \frac{2\sum_{\ell}d_{\ell}}{N^2}$$
  Moreover, we can find a probabilistic bound for $\|\nu_1\|$ too by a variation of \eqref{eq:relax_constraint} for the random case, which now becomes
\begin{equation}\label{eq:rand_relax_constraint}
  y_i^T\nu_1+(\mathbb{P}_{\mathcal{S}_{\ell}} z_i)^T\nu_1 \leq 1-z_i^T\nu_2\leq 1+\delta_2\|\nu_2\||\cos{\angle(z_i,\nu_2)}|.
\end{equation}
  Substituting the upper bound of the cosines to \eqref{eq:showing_dual_sep_cond_rand} and \eqref{eq:rand_relax_constraint}, we get respectively
  \begin{equation*}
    |\langle x, \nu \rangle| \leq \mu\|\nu_1\| + \|y\|\|\nu_2\|\sqrt{\frac{6\log N}{n-d_{\ell}}} + \|z\|\|\nu\|\sqrt{\frac{6\log N}{n}},
  \end{equation*}
and
$$
\|\nu_1\|\leq \frac{1+\delta\|\nu_2\|\sqrt{\frac{6\log N}{n}}}{r(\mathcal{Q}_{-i}^{\ell})-\delta_1}.
$$
This new bound of $\|\nu_1\|$ follows from \eqref{eq:rand_relax_constraint}, Lemma~\ref{lemma:circum_inradius}~and~\ref{lemma:Y_containing_set}. For the bound of $\|\nu_2\|$ we simply use \eqref{eq:nu2_bound}:
$$\|\nu_2\|\leq \lambda  \delta \left(\frac{1}{r(\mathcal{Q}_{-i}^{\ell})}+1\right).$$

To lighten notations in this proof, denote $$r:=r(\mathcal{Q}_{-i}^{\ell}), \quad\epsilon := \sqrt{\frac{6\log N}{n-\max_{\ell}{d_{\ell}}}},\quad \mu:=\mu(\mathcal{X}_{\ell}).$$
Substitute them in the bound, we get
  \begin{align}
    |\langle x, \nu \rangle| \leq& \frac{\mu+\delta\epsilon}{r-\epsilon \sqrt{d_{\ell}}\delta}+\frac{\lambda\delta^2(\mu+\delta\epsilon)\epsilon}{r-\epsilon \sqrt{d_{\ell}}\delta}\left(\frac{1}{r}+1\right)
    +\lambda \delta \epsilon \left(\frac{1}{r}+1\right)
    + \lambda \delta^2\epsilon\left(\frac{1}{r}+1\right)\nonumber\\
    =& \frac{\mu+\delta\epsilon}{r-\epsilon \sqrt{d_{\ell}}\delta}+
    \frac{\lambda\epsilon\delta^2(\frac{\mu+\delta\epsilon}{r}) + \lambda\epsilon\delta^2(\mu+\delta\epsilon)}{r-\epsilon \sqrt{d_{\ell}}\delta}+\frac{\lambda\delta(\delta+1)\epsilon}{r}+\lambda\delta(\delta+1)\epsilon\nonumber\\
    \explain{\leq}{*}&\frac{\mu+\delta\epsilon}{r-\epsilon \sqrt{d_{\ell}}\delta}+
    \frac{\lambda \epsilon \delta^2}{r-\epsilon \sqrt{d_{\ell}}\delta} + \frac{\lambda \epsilon\delta^2}{r-\epsilon \sqrt{d_{\ell}}\delta} + \frac{\lambda \epsilon(\delta+\delta^2)}{r-\epsilon \sqrt{d_{\ell}}\delta} + \lambda\epsilon (\delta+\delta^2)
     \nonumber\\
     =& \frac{\mu+\delta\epsilon + \lambda \epsilon(\delta+3\delta^2)}{r-\epsilon \sqrt{d_{\ell}}\delta}+ \lambda \epsilon (\delta+\delta^2) \explain{\leq}{**} \frac{\mu+\delta\epsilon + 3\rho \epsilon}{r-\epsilon \sqrt{d_{\ell}}\delta}+ \rho \epsilon.\label{eq:derive_dual_sep_rand}
  \end{align}
In the inequality ``$*$'', we used $(\mu+\epsilon\delta)/r <1$ and $\mu + \epsilon\delta <1$; and in the inequality ``$**$'', we used $\lambda \delta^2 \leq \lambda\epsilon (\delta +\delta^2)$ and replaced all such expression with $\rho\epsilon$ that we defined earlier.


  Now impose the dual detection constraint on the upper bound, we get:
  $$ \rho\epsilon + \frac{\mu +\delta\epsilon + 3\rho\epsilon}{r-\delta\sqrt{d_\ell}\epsilon} < 1.$$
  Reorganized the inequality, we reach the desired condition:
  $$ \mu  +\delta\epsilon< (1-\rho\epsilon)(r-\delta\sqrt{d_{\ell}}\epsilon)  -3\rho\epsilon. $$
  There are $N^2$ instances for each of the three events related to the consine value, apply union bound we get the failure probability $\frac{6}{N}+\frac{2\sum_{\ell}{d_{\ell}}}{N^2} \leq \frac{8}{N}$. Note $\sum_{\ell} d_{\ell} \leq N$ because one needs at least $d_{\ell}$ data points to span an $d_{\ell}$ dimensional subspace. This concludes the proof.
\end{proof}


\begin{lemma}[Avoid trivial solution under random noises]\label{lemma:avoid_trivial_random}
Let $\epsilon = \sqrt{\frac{6\log N}{n}}$ and assume $\min_\ell r_\ell - 2\epsilon\delta-\epsilon\delta^2>0$.  If we take
\begin{equation}\label{eq:avoid_trivial_random}
\lambda > (r_\ell - 2\epsilon\delta-\epsilon\delta^2)^{-1}
 \end{equation}
 for every $\ell=1,...,n$, then the solution $c_i\neq 0$ for all $i$ with probability at least $1-6/N^2$.
\end{lemma}
\begin{proof}
We use the same argument as in Section~\ref{sec:avoid_trivial}, except that we now have a tighter probabilistic bound for \eqref{eq:lowerbounding_max_affinity}. For any $i$, $k$ in the equation, $z_i$ and $z_k$ are independent to each other and to $y_k$, $y_i$ respectively. Therefore, we can invoke Lemma~\ref{lemma:spherical_cap} and obtain
$$\left|\langle y_k, z_i\rangle + \langle z_k, y_i\rangle + \langle z_k, z_i\rangle\right| \leq 2\epsilon \delta + \epsilon \delta^2,$$
with probability greater than $2/N^3$.
The proof is complete by taking the union bound over all $3 \sum_i N_\ell = 3N$ instances.
\end{proof}

\subsection{Proof of Theorem~\ref{thm:thm_random_noise} for Deterministic Data and Random Noise}
We now prove Theorem~\ref{thm:thm_random_noise}. Lemma~\ref{lemma:dual_sep_random} has already provided the separation condition. The things left are to find the range of $\lambda$ and update the condition of $\delta$.

\noindent\textbf{The range of $\lambda$:}
The range of valid $\lambda$ for the random noise case can be obtained by substituting $\delta_1<\delta\sqrt{d_{\ell}}\epsilon$ in \eqref{eq:avoid_trivial_random} and rewriting \eqref{eq:dual_sep_random} with respect to $\lambda$. This gives us
\begin{equation}\label{eq:lambda_range_rand}
\frac{1}{r- 2\epsilon \delta-\epsilon\delta^2}<
        \lambda<\frac{r_{\ell}-\mu_{\ell}-\delta\epsilon - \delta \sqrt{d_{\ell}} \epsilon}{\epsilon\delta(1+\delta)(3+r_{\ell}-\delta\sqrt{d_{\ell}}\epsilon)}.
\end{equation}

  We remark that acritical difference from the deterministic noise model is that now there is a small $\epsilon$ in the denominator of the upper endpoint of the interval. Assume small $\mu$, the valid range of $\lambda$ expands to an order of $\Theta(1/r)\leq\lambda\leq \Theta(r/(\epsilon\max\{\delta^2,\delta\}))$.


\noindent\textbf{The condition of $\delta$:}
Now we will show that the two conditions
\begin{align*}
 \epsilon\delta<\min_{\ell}\frac{r_{\ell}-\mu_{\ell}}{2\sqrt{d_{\ell}}+2}, &&\text{and}&& \epsilon\delta(1+\delta) < \min_{\ell}\frac{r(r_\ell-\mu_\ell)}{4r_\ell+6},
\end{align*}
stated in the Theorem~\ref{thm:thm_random_noise} are sufficient for the three inequalities
\begin{empheq}[left=\empheqlbrace]{align}
    &r_\ell-\mu_\ell - \delta\epsilon >0;\label{eq:rand_ineq1}\\
  &r-2\delta\epsilon -\epsilon\delta^2 >0;\label{eq:rand_ineq2}\\
&\frac{1}{r- 2\epsilon \delta-\epsilon\delta^2}<\frac{r_{\ell}-\mu_{\ell}-\delta\epsilon - \delta \sqrt{d_{\ell}} \epsilon}{\epsilon\delta(1+\delta)(3+r_{\ell}-\delta\sqrt{d_{\ell}}\epsilon)};\label{eq:rand_ineq3}
\end{empheq}
to hold for $\ell=1,...,L$. Note that we used \eqref{eq:rand_ineq1} in \eqref{eq:derive_dual_sep_rand} when we derive the dual separation condition and  \eqref{eq:rand_ineq2} is assumed in Lemma~\ref{lemma:avoid_trivial_random}, lastly \eqref{eq:rand_ineq3} ensures a valid $\lambda$ to exist in \eqref{eq:lambda_range_rand}.
Inequality \eqref{eq:rand_ineq1} and \eqref{eq:rand_ineq2} hold trivially given the two conditions, it remains to show \eqref{eq:rand_ineq3}:
\begin{align*}
\delta(1+\delta) < \frac{r(r_\ell-\mu_\ell)}{\epsilon(4r_\ell+6)}
\Leftrightarrow&\; \epsilon\delta(1+\delta) (2r_\ell + 3) < \frac{r(r_\ell-\mu_\ell)}{2}\\
\Rightarrow &\; \epsilon\delta(1+\delta) (r_\ell -\mu_\ell + r_\ell + 3  ) <  \frac{r(r_\ell-\mu_\ell)}{2}\\
\Leftrightarrow &\; \epsilon\delta(1+\delta) ( r_\ell + 3  ) + 2\epsilon \delta(1+\delta)\frac{r_\ell -\mu_\ell}{2} < \frac{r(r_\ell-\mu_\ell)}{2}\\
\Leftrightarrow& \;  \frac{1}{r-2\epsilon\delta -2\epsilon\delta^2}<\frac{r_\ell -\mu_\ell}{2\epsilon\delta(1+\delta) ( r_\ell + 3  )}\\
\Rightarrow &\; \frac{1}{r-2\epsilon\delta -\epsilon\delta^2}<\frac{r_\ell -\mu_\ell}{2\epsilon\delta(1+\delta) ( r_\ell + 3 -\delta\sqrt{d_\ell}\epsilon )}
\stackrel{(a)}{\Rightarrow}  \eqref{eq:rand_ineq3},
\end{align*}
where (a) holds by applying the first condition.
This concludes the proof for Theorem~\ref{thm:thm_random_noise}.

\subsection{Proof of Theorem~\ref{thm:semirandom}  for the Semirandom Model with Random Noise}
To prove Theorem~\ref{thm:semirandom}, we only need to bound the inradii $r$ and the incoherence parameter $\mu$ under the new assumptions, then plug them into Theorem~\ref{thm:thm_random_noise}.

\begin{lemma}[Inradius bound of random samples]\label{lemma:random_inradius}
  In random sampling setting, when each subspace is sampled $N_{\ell}=\kappa_{\ell} d_{\ell}$ data points randomly, we have:
  \begin{align*}
    Pr\left\{c(\kappa_{\ell})\sqrt{\frac{
    \beta\log{(\kappa_{\ell})}}{ d_{\ell}}}\leq r(\mathcal{Q}^{(\ell)}_{-i})\text{ for all pairs }(\ell,i)\right\}
    \geq 1-\sum_{\ell=1}^{L}N_{\ell}e^{-d_{\ell}^{\beta}N_{\ell}^{1-\beta}}
  \end{align*}
\end{lemma}
This is extracted from Section-7.2.1 of \citet{soltanolkotabi2011geometric}. $\kappa_{\ell}=(N_{\ell}-1)/{d_\ell}$ is the relative number of iid samples. $c(\kappa)$ is some positive value for all $\kappa>1$ and for a numerical value $\kappa_0$, if $\kappa>\kappa_0$, we can take $c(\kappa)=\frac{1}{\sqrt{8}}$. Take $\beta=0.5$, we get the required bound of $r$ in Theorem~\ref{thm:semirandom}.

Now, we   provide a probabilistic upper bound of the projected subspace incoherence condition under the semi-random model by adapting Lemma~7.5 of \citet{soltanolkotabi2011geometric} into our new setup.

\begin{lemma}[Incoherence bound]\label{lemma:deterministic_incoherence}
In deterministic subspaces/random sampling setting, the subspace incoherence is bounded from above:
\begin{align*}
   Pr\Bigl\{&\mu(\mathcal{X}_{\ell})\leq t\left(\log[(N_{\ell}+1)N_{\ell^\prime}]+\log L\right)\frac{\mathrm{aff}(S_{\ell},S_{\ell^\prime})}{\sqrt{d_{\ell}}\sqrt{d_{\ell^\prime}}}\\
   &\text{ for all pairs}(\ell,\ell^\prime)\text{ with }\ell\neq \ell^\prime\Bigr\}
   \geq 1- \frac{1}{L^2}\sum_{\ell\neq \ell^\prime} \frac{1}{(N_{\ell}+1)N_{\ell^\prime}}e^{-\frac{t}{4}}.
\end{align*}
\end{lemma}
\begin{proof}
The proof is an extension of a similar proof in \citet{soltanolkotabi2011geometric}.
First we will show that when noise $z_i^{(\ell)}$ is spherical symmetric, and clean data points $y_i^{(\ell)}$ has iid uniform random direction, projected dual directions $v_i^{(\ell)}$ also follows a uniform random distribution.

Now we  prove the claim. First by definition,
$$ v_i^{(\ell)} = v(x_i^{(\ell)},X_{-i}^{(\ell)},\mathcal{S}_{\ell},\lambda)=\frac{\mathbb{P}_{S_{\ell}}\nu}{\|\mathbb{P}_{S_{\ell}}\nu\|} = \frac{\nu_1}{\|\nu_1\|}.$$
Recall that $\nu$ is the unique optimal solution of $\mathbf{D}_1$ \eqref{eq:dual_fictitious2}. Fix $\lambda$, $\mathbf{D}_1$ depends on two inputs, so we denote $\nu(x,X)$ and consider $\nu$ a function. Moreover, $\nu_1=\mathbb{P}_{\mathcal{S}}\nu$ and $\nu_2=\mathbb{P}_{\mathcal{S}^{\perp}}\nu$. Let $U\in n\times d$ be a set of orthonormal basis of $d$-dimensional subspace $\mathcal{S}$ and a rotation matrix $R\in \mathbb{R}^{d\times d}$. Then rotation matrix within subspace is hence $URU^T$. Let
\begin{align*}
  x_1 :=& \mathbb{P}_{\mathcal{S}} x= y+z_1 \sim URU^Ty + URU^Tz_1,\\
  x_2 :=& \mathbb{P}_{\mathcal{S}^{\perp}} x = z_2.
\end{align*}
As $y$ is distributed uniformly on the unit sphere of $\mathcal{S}$, and $z$ is a spherical symmetric noise (hence $z_1$ and $z_2$ are also spherical symmetric in subspace), for any fixed $\|x_1\|$, the distribution is uniform on the sphere, namely the conditional distribution $Pr\left(x_1\middle|\|x_1\|=\alpha\right)$ is uniform on the sphere with radius $\alpha$. It suffices to show that with fixed $\|x_1\|$, $v$ (the unit direction of projected dual variable $\nu_1$) also follows a uniform distribution on a unit sphere of the subspace.

Since inner product $\langle x,\nu\rangle = \langle x_1,\nu_1\rangle + \langle x_2,\nu_2\rangle$, we argue that if $\nu$ is the optimal solution of
\begin{align*}
\quad \max_{\nu} \; \langle x,\nu \rangle - \frac{1}{2\lambda}\nu^T\nu,\quad\text{subject to:}\quad &\|X^T\nu\|_{\infty} \leq 1,
\end{align*}
then the optimal solution of the following optimization
\begin{align*}
 &\max_{\nu} \; \langle URU^Tx_1+x_2,\nu \rangle - \frac{1}{2\lambda}\nu^T\nu,\\
&\text{subject to:}\quad \|(URU^TX_1+X_2)^T\nu\|_{\infty} \leq 1,
\end{align*}
is indeed the transformed $\nu$ under the same $R$, i.e.,
\begin{align}
  \nu(R)&=\nu(URU^Tx_1+x_2, URU^TX_1+X_2)\nonumber\\
&=URU^T\nu_1(x,X)+\nu_2(x,X)=URU^T\nu_1 + \nu_2.\label{eq:nu(R)}
\end{align}
To verify the argument, check that $\nu^T\nu = \nu(R)^T\nu(R)$ and
\begin{align*}
\langle URU^Tx_1+x_2,\nu(R)\rangle = \langle URU^Tx_1,URU^T\nu_1\rangle + \langle x_1,\nu_2\rangle
= \langle x,\nu\rangle
\end{align*}
for all inner products in both objective function and constraints, preserving the optimality.

By projecting \eqref{eq:nu(R)} to subspace, we show that operator $v(x,X,S)$ is linear \textit{vis a vis} subspace rotation $URU^T$, i.e.,
\begin{equation}\label{eq:v_linearity}
  v(R) = \frac{\mathbb{P}_{S_{\ell}}\nu(R)}{\|\mathbb{P}_{S_{\ell}}\nu(R)\|} = \frac{URU^T\nu_1}{\|URU^T\nu_1\|}=URU^Tv.
\end{equation}

On the other hand, we know that
\begin{align*}
  URU^Tx_1+x_2 \sim x_1+x_2, &&
  URU^TX_1+X_2 \sim X_1+X_2,
\end{align*}
where $A\sim B$ means that the random variables $A$ and $B$ follows the same distribution.
This is because when
$\|x_1\|$ is fixed and each columns in $X_1$ has fixed magnitudes, $URU^Tx_1 \sim x_1$ and $URU^TX_1 \sim X_1$. Also, adding additional random variables $x_2$ and $X_2$ changes the distribution the same way on both sides. Therefore,
\begin{equation}\label{eq:v_distribution}
  v(R) = v(URU^Tx_1+x_2, URU^TX_1+X_2, \mathcal{S})\sim v(x, X, \mathcal{S}).
\end{equation}

Combining \eqref{eq:v_linearity} and \eqref{eq:v_distribution}, we conclude that for any rotation $R$
$$v_i^{(\ell)}(R)\sim URU^Tv_i^{(\ell)}.$$
In other words, the distribution of $v_i^{(\ell)}$ is uniform on the unit sphere of $\mathcal{S}_{\ell}$.


After this key step, the rest is identical to the proof of Lemma~7.5 of \citet{soltanolkotabi2011geometric}. The idea is to use Lemma~\ref{lemma:spherical_cap} (upper bound of area of spherical caps) to provide a probabilistic bound of the pairwise inner product and Borell's inequality to show the concentration around the expected cosine canonical angles, namely, $\|{U^{(k)}}^TU^{(\ell)}\|_F/\sqrt{d_{\ell}}$. The proof is standard so we omit it in this paper.
\end{proof}

\subsection{Proof of Theorem~\ref{thm:fullrandom} for the Fully Random Model with Gaussian Noises}
The proof of Theorem~\ref{thm:fullrandom} essentially applies Theorem~\ref{thm:thm_random_noise} with specific inradii bound and incoherence bound. The bound for inradius is given in Lemma~\ref{lemma:random_inradius} and we use the following Lemma extracted from Step~2 of Section~7.3 in \citet{soltanolkotabi2011geometric} to bound the subspace incoherence.
\begin{lemma}[Incoherence bound of random subspaces]\label{lemma:random_incoherence}
In random subspaces setting, the projected subspace incoherence is bounded from above:
\begin{equation*}
   Pr\left\{\mu(\mathcal{X}_{\ell})\leq \sqrt{\frac{6\log N}{n}}\text{ for all }\ell\right\}\geq 1- \frac{2}{N}.
\end{equation*}
\end{lemma}
Now that we have shown that the projected dual directions are randomly distributed in their respective subspace, together with the fact that the subspaces themselves are randomly generated, we conclude that all clean data points $y$ and projected dual direction $v$ from different subspaces can be considered iid generated from the ambient space. The proof of Lemma~\ref{lemma:random_incoherence} follows by simply applying Lemma~\ref{lemma:spherical_cap} and a union bound across all $N^2$ events.

\section{Experiments}\label{sec:experiments}
To demonstrate the practical implications of our robustness guarantees for LASSO-SSC, we conduct four numerical experiments including three with synthetic generated data and one with real data. For fast computation, we use ADMM implementation of LASSO solver\footnote{Freely available at:\\
\href{http://www.stanford.edu/~boyd/papers/admm/}{http://www.stanford.edu/{\textasciitilde}boyd/papers/admm/}} with default numerical parameters. Its complexity is proportional to the problem size and the convergence guarantee \citep{boyd2011admm}. We also implement a simple ADMM solver for the matrix version SSC
\begin{equation}\label{eq:MatrixLasso}
\begin{aligned}
\min_{C} \; \|C\|_1+\frac{\lambda}{2}\|X-XC\|_F^2 \quad
\text{s.t.} \;\quad\mathrm{diag}(C)=0,
\end{aligned}
\end{equation}
which is consistently faster than the column-by-column LASSO version. This algorithm is first described in \citet{elhamifar2012ssc_journal}. To be self-contained, we   provide the pseudocode and some numerical simulation in the appendix.

\subsection{Numerical simulation}
Our three numerical simulations test the effects of noise magnitude $\delta$, subspace rank $d$ and number of subspace $L$ respectively.

\textbf{Methods}:
 To test our methods for all parameters, we scan through an exponential grid of $\lambda$ ranging from $\sqrt{n}\times 10^{-2}$ to $\sqrt{n}\times10^3$. In all experiments, ambient dimension $n=100$, relative sampling $\kappa=5$, subspace and data are drawn uniformly at random from unit sphere and then corrupted by Gaussian noise $Z_{ij}\sim N(0,\sigma/\sqrt{n})$. We measure the success of the algorithm by the relative violation of Self-Expressiveness Property defined below.
\begin{align*}
\mathrm{RelViolation}\left(C,\mathcal{M}\right)= \frac{\sum_{(i,j)\notin \mathcal{M}}|C|_{i,j}}{\sum_{(i,j)\in \mathcal{M}}|C|_{i,j}}
\end{align*}
where $\mathcal{M}$ is the ground truth mask containing all $(i,j)$ such that $x_i, x_j \in \mathcal{X}^{(\ell)}$ for some $\ell$. Note that $\mathrm{RelViolation}\left(C,\mathcal{M}\right)=0$ implies that SEP is satisfied. We also check that there is no all-zero columns in $C$, and the solution is considered trivial otherwise.

\textbf{Results}:
The simulation results confirm our theoretical findings. In particular, Figure~\ref{fig:Exp1_noise} shows that LASSO subspace detection property is possible for a very large range of $\lambda$ and the dependence on noise magnitude is roughly $1/\sigma$ as predicted in \eqref{eq:thm_rand_noise_lambda_range}. In addition, the sharp contrast of Figure~\ref{fig:Exp2_rank} and \ref{fig:Exp3_L} demonstrates our observations on the sensitivity of $d$ and $L$. 


\begin{figure}[htb]
\begin{minipage}[t]{0.46\linewidth}
  \centering
  \includegraphics[width=\textwidth]{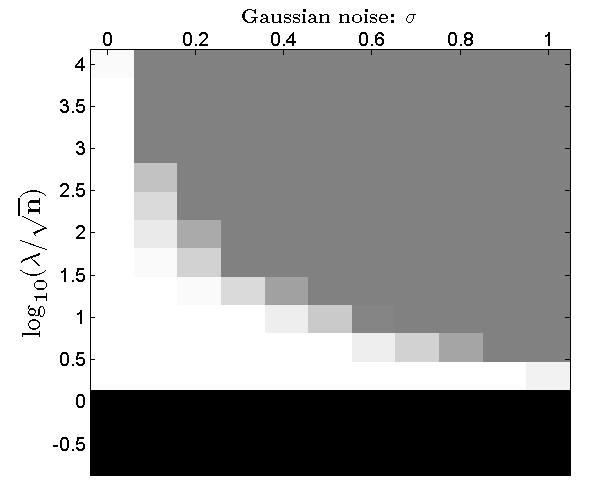}\\
  \caption[Exact recovery vs. increasing noise.]{Exact recovery under noise. Simulated with $n=100, d=4, L=3, \kappa=5$ with increasing Gaussian noise $N(0,\sigma/\sqrt{n})$. \textbf{ Black:} trivial solution ($C=0$); \textbf{Gray:} $\mathrm{RelViolation}>0.1$; \textbf{White:} $\mathrm{RelViolation}=0$.  }\label{fig:Exp1_noise}
\end{minipage}
\hspace{0.3cm}
\begin{minipage}[t]{0.53\linewidth}
  \centering
  \includegraphics[width=\textwidth]{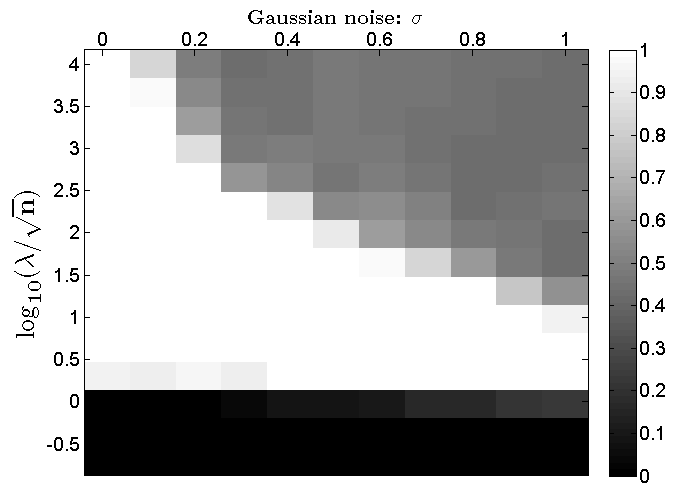}\\
  \caption[Spectral clustering accuracy vs. increasing noise.]{Spectral clustering accuracy for the experiment in Figure~\ref{fig:Exp1_noise}. The rate of accurate classification is represented in grayscale. White region means perfect classification. It is clear that exact subspace detection property (Definition~\ref{def:lasso_detection}) is not necessary for perfect classification. }\label{fig:Exp1_acc_map}
\end{minipage}
\end{figure}

\begin{figure}[ht]

\begin{minipage}[t]{0.48\textwidth}
  \centering
  \includegraphics[width=\textwidth]{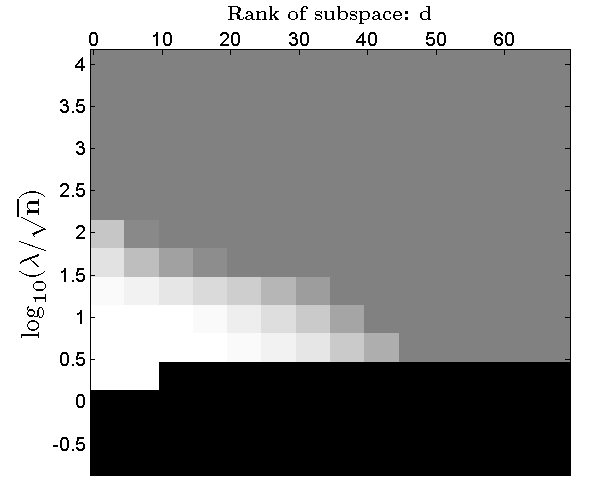}\\
  \caption[Effects of cluster rank $d$.]{Effects of cluster rank $d$. Simulated with $n=100, L=3, \kappa=5, \sigma=0.2$ with increasing $d$. \textbf{Black:} trivial solution ($C=0$); \textbf{Gray:} $\mathrm{RelViolation}>0.1$; \textbf{White:} $\mathrm{RelViolation}=0$.  Observe that beyond a point, subspace detection property is not possible for any $\lambda$. }\label{fig:Exp2_rank}
\end{minipage}
\hspace{0.02\textwidth}
\begin{minipage}[t]{0.48\textwidth}
  \centering
  \includegraphics[width=\textwidth]{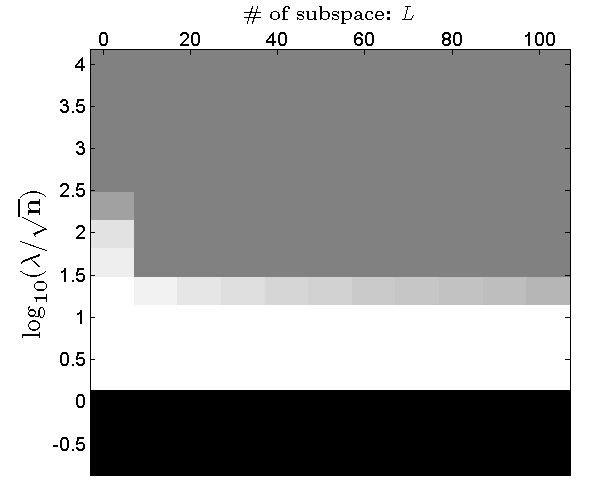}\\
  \caption[Effects of number of subspace $L$.]{Effects of number of subspace $L$. Simulated with $n=100, d=2, \kappa=5, \sigma=0.2$ with increasing $L$. \textbf{Black:} trivial solution ($C=0$); \textbf{Gray:} $\mathrm{RelViolation}>0.1$; \textbf{White:} $\mathrm{RelViolation}=0$.  Note that even at the point when $dL=200$(subspaces are highly dependent), subspace detection property holds for a large range of $\lambda$. }\label{fig:Exp3_L}
\end{minipage}

\end{figure}

\subsection{Face Clustering Experiments}
In this section, we evaluate the noise robustness of with LASSO-SSC on Extended YaleB \citep{lee2005extendedyaleb}, a real life face dataset of 38 subjects. For each subject, 64 face images are taken under different illuminations.


\textbf{Subspace Modeling of Face Images}:
According to \citet{basri2003lambertianface}, face images under different illuminations can be well-approximated by a 9-dimensional linear subspace. In addition, \citet{zhou2007PhotometricFace} reveals the underlying 3-dimensional subspace structure by assuming Lambert's reflectance and ignoring the shadow pixels. For the physics of this subspace model, we refer the readers to \citet{basri2003lambertianface} and \citet{zhou2007PhotometricFace} for detailed explanations.

\textbf{Method}:  We conduct face clustering experiments on Extended YaleB Dataset with both 9D and 3D representations of face images and compare them under varying number of subspaces $L$ and different level of injected noise. Specifically, the 9D subspaces are generated by projecting the data matrix corresponding to each subject to a 9D subspace via PCA and the 3D subspaces are generated by a factorization-based robust matrix completion method. Then we scan through a random selection of $[2,4,6,10,12,18,38]$ subjects and for each experiment we inject additive Gaussian noise $N(0,\sigma/\sqrt{n})$ with $\sigma=[0,0.01,...,0.99,1]$\footnote{In order to compare the effect of varying level of noise, we choose to inject artificial noise in this experiment. The performance of LASSO-SSC on real noise/data corruptions is well-documented in the motion segmentation experiments of \citet{elhamifar2009ssc,elhamifar2012ssc_journal}}. Each photo is resized to $48\times 42$ so we have ambient dimension $n=2016$ and there are $64$ sample points for each subspace, hence $N=64L$. The parameter $\lambda$ is not carefully tuned, but simply chosen to be $\sqrt{n}$, which is order-wise correct for small $\sigma$ according to \eqref{eq:thm_rand_noise_lambda_range}. 

\textbf{Results}:
As we can see in Figure~\ref{fig:Rank3Face}~and~\ref{fig:Rank9Face}, the range where LASSO-Subspace Detection Property holds is much larger for the rank-$3$ experiments than the rank-$9$ experiments. Also, the recovery is not sensitive to the number of faces we want to cluster. Indeed, LASSO-SSC is able to succeed for both rank-9 and rank-3 data with a considerable range of noise even for the full 38 subjects clustering task.

These observations confirm our theoretical and simulation results---on deterministic subspace data from a real-life problem---that noise robustness of LASSO-SSC is sensitive to the subspace dimension $d$ but  not the number of subspaces $L$.


\begin{figure}[ht]
\begin{minipage}[t]{0.48\textwidth}
  \centering
  \includegraphics[width=\textwidth]{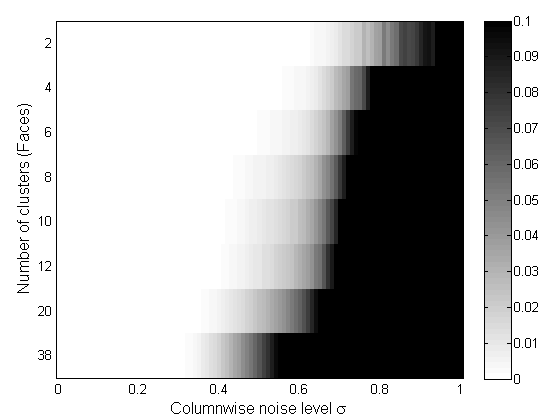}\\
  \caption[Effects of number of subspace $L$.]{RelViolation of the Face clustering experiments with Rank 3 photometric face. }\label{fig:Rank3Face}
\end{minipage}
\hspace{0.02\textwidth}
\begin{minipage}[t]{0.48\textwidth}
  \centering
  \includegraphics[width=\textwidth]{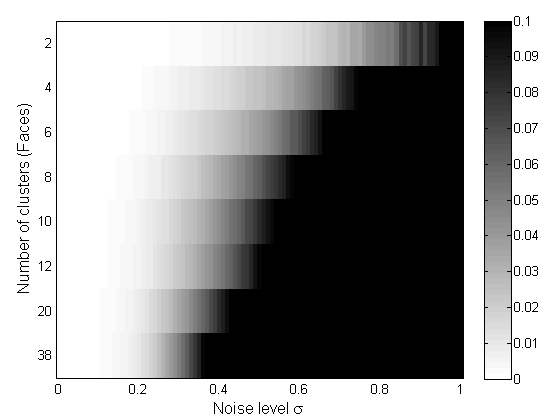}\\
  \caption[Effects of cluster rank $d$.]{RelViolation of the Face clustering experiments with Rank 9 faces (after projection).  }\label{fig:Rank9Face}
\end{minipage}
\end{figure}

\section{Conclusion and Future Directions}\label{sec:conclusion}
We presented a theoretical analysis for noisy subspace segmentation problem that is of great practical interests. We showed that the popular SSC algorithm {\em exactly} (not approximately) detects the subspaces in the noisy case, which justified its empirical success on real problems. Our results are the first in showing LASSO-SSC to work under deterministic data and noise. For stochastic noise, we show that LASSO-SSC works even when noise is much larger than the signal. In addition, we discovered the orderwise relationship between LASSO-SSC's robustness to the level of noise and the subspace dimension, and we found that robustness is insensitive to the number of subspaces.
These results lead to new theoretical understanding of SSC, and
 provide guidelines for practitioners and application-level researchers to judge whether SSC could possibly work well for their respective applications.

Open problems for subspace clustering include the graph connectivity problem raised by \citet{nasihatkon2011graph}, missing data problem (a first attempt is performed in \citet{eriksson2011high_rankMC}, which requires an unrealistic number of data), sparse corruptions on data and others. One direction closely related to this paper is to introduce  a more practical metric of success. As we illustrated in the paper, subspace detection property is not necessary for perfect clustering. In fact from a pragmatic point of view, even perfect clustering is not necessary. Typical applications allow for a small number of misclassifications. It would be interesting to see whether stronger robustness results can be obtained for a more practical metric of success.


\acks{H. Xu was supported by the Ministry of Education of Singapore through AcRF Tier Two grant R-265-000-443-112.}



\appendix
\section{Differences to \citet{soltanolkotabi2013robust}}
As we reviewed above, \citet{soltanolkotabi2013robust} analyzed almost the same algorithm under the semi-random model. Besides the comparisons we made in Section~\ref{sec:RelatedWorks} regarding the model of analysis and allowable noise level, there are a few other minor differences which we list here.
\begin{description}
  \item[``Non-trivial'' v.s. ``Many true discoveries''.] In LASSO-Subspace Detection Property, we only require the resulting regression coefficient to be non-zero, while \citet[Theorem~3.2]{soltanolkotabi2013robust} has a result showing the conditions under which the number of non-zero coefficient is a constant fraction of subspace dimension $d$. Our results are weaker but more general (works without the semi-random assumption).

      In fact, the conditions are more similar than different. We both pick regression coefficient of the same order in the semi-random model. In addition, when $d$ becomes smaller than $\log \rho(i)/c_0$,   these two conditions are essentially the same.

  \item[Choosing parameter $\lambda$.]
        \citet{soltanolkotabi2013robust} provides a two-pass mechanism to adaptively tune the parameter $\lambda$ for each Lasso-SSC and the results are proven for this particular $\lambda$.
        On the other hand, our results are stated for any $\lambda$ in a specified range. The choice of $\lambda$ can also be independently tuned for each Lasso-SSC.

      In practice, it is advisable to choose $\lambda$ slightly larger than what is required for it to be non-trivial. We described two strategies in the discussion underneath Theorem~\ref{thm:thm_general}.

  \item[Proof techniques.] The proofs are admittedly similar in many ways (since we solve the same problem!), but the key technical components in controlling the magnitude of dual variables $\nu_1$ and $\nu_2$ are different. \citet[Lemma~8.5]{soltanolkotabi2013robust} relies on the semi-random model, and the resulting restricted isometry property (of the block of data points corresponding to each subspace). In contrast, our bound for $\|\nu_2\|$ does not require any probabilistic assumptions, therefore more general. 
  It is however looser than \citet[Lemma~8.5]{soltanolkotabi2013robust} by a factor of $\sqrt{d}$ when we specialized in the semi-random model.
  This is probably what led to our worse dependence on the subspace dimension $d$ in the bound we described in the discussion of Theorem~\ref{thm:semirandom}.
\end{description}
In conclusion, we find that our results are complementary to that in \citet{soltanolkotabi2013robust} and provide a novel point of view to the theoretical analysis for subspace clustering problems.

\section{Numerical algorithm to solve Matrix-LASSO-SSC}\label{sec:ADMM_matrix-LASSO-SSC}
In this section we outline the steps of solving the matrix version of LASSO-SSC below. Note that \citet{elhamifar2012arxiv} derived a more general version of Matrix-LASSO-SSC  to account for not only noisy but also sparse corruptions. We include this appendix merely for the convenience of the readers. Consider
\begin{equation}\label{eq:MatrixLasso}
\begin{aligned}
\min_{C} \; &\|C\|_1+\frac{\lambda}{2}\|X-XC\|_F^2 \quad
\text{s.t.} \;\quad\mathrm{diag}(C)=0.
\end{aligned}
\end{equation}
While this convex optimization problem can be solved by some off-the-shelf general purpose solvers such as SeDuMi~\citep{sturm1999sedumi} or SDPT3 \citep{toh1999sdpt3}, such approaches are usually slow and non-scalable. An ADMM \citep{boyd2011admm} version of the problem is described here for fast computation. It solves an equivalent optimization program
\begin{equation}\label{eq:MatrixLasso_modify}
\begin{aligned}
\min_{C} \; &\|C\|_1+\frac{\lambda}{2}\|X-XJ\|_F^2 \\
\text{s.t.} \;&\quad J=C-\mathrm{diag}(C).
\end{aligned}
\end{equation}
We add to the Lagrangian with an additional quadratic penalty term for the equality constraint and get the augmented Lagrangian
\begin{align*}
\mathcal{L}=& \|C\|_1+\frac{\lambda}{2}\|X-XJ\|_F^2 + \frac{\mu}{2}\|J-C+\mathrm{diag}(C)\|_F^2
+tr(\Lambda^T(J-C+\mathrm{diag}(C))),
\end{align*}
where $\Lambda$ is the dual variable and $\mu$ is a parameter. Optimization is done by alternatingly optimizing over $J$, $C$ and $\Lambda$ until convergence. The update steps are derived by solving $\partial \mathcal{L}/\partial J=0$ and $\partial \mathcal{L}/\partial C=0$. Notice that the objective function is non-differentiable for $C$ at origin so we use the now standard soft-thresholding operator~\citep{donoho1995noising}. For both variables, the solution is given in closed-forms. For the update of $\Lambda$, we simply use the gradient descent method. For details of the ADMM algorithm and its guarantee, please refer to \citet{boyd2011admm}. To accelerate the convergence, it is possible to introduce a parameter $\rho$ and increase $\mu$ by $\mu=\rho\mu$ at every iteration. The full algorithm is summarized in Algorithm~\ref{alg:MatrixSSC}.

\begin{algorithm}[tb]
   \caption{Matrix-LASSO-SSC}
   \label{alg:MatrixSSC}
\begin{algorithmic}
   \STATE {\bfseries Input:}
   Data points as columns in $X\in \mathbb{R}^{n\times N}$, tradeoff parameter $\lambda$, numerical parameters $\mu_0$ and $\rho$.
   \STATE Initialize $C=0$, $J=0$, $\Lambda=0$, $k=0$.
   \WHILE{not converged}
   \STATE{1. } Update $J$ by
   $$J=(\lambda X^TX+\mu_k I)^{-1}(\lambda X^TX+\mu_k C-\Lambda).$$
   \STATE{2. } Update $C$ by
   $$ C^{'}=\mathrm{SoftThresh}_{\frac{1}{\mu_k}}\left(J+\Lambda/\mu_k\right), $$
   $$ C=C^{'}-\mathrm{diag}(C^{'}).$$
   \STATE{3. } Update $\Lambda$ by
   $$\Lambda=\Lambda+\mu_k(J-C)$$
   \STATE{4. } Update parameter $\mu_{k+1}=\rho\mu_k.$
   \STATE{5. } Iterate $k=k+1$;
   \ENDWHILE
   \STATE {\bfseries Output:} Affinity matrix $W=|C|+|C|^T$
\end{algorithmic}
\end{algorithm}

Note that for the special case when $\rho=1$, the inverse of $(\lambda Y^TY+\mu I)$ can be pre-computed, and hence each iteration can be computed in linear time. Empirically, we found it good to set $\mu=\lambda$ and it takes roughly 50-100 iterations to converge to a sufficiently good points. We remark that the matrix version of the algorithm is much faster than the column-by-column ADMM-Lasso and achieves almost the same numerical accuracy; see  our experiments in Figure~\ref{fig.runtime_data},\ref{fig.obj_data},\ref{fig.runtime_dim}~and~\ref{fig.obj_dim}.

\begin{figure}[htb]
\begin{minipage}[t]{0.48\linewidth}
  \centering
  \includegraphics[width=8cm]{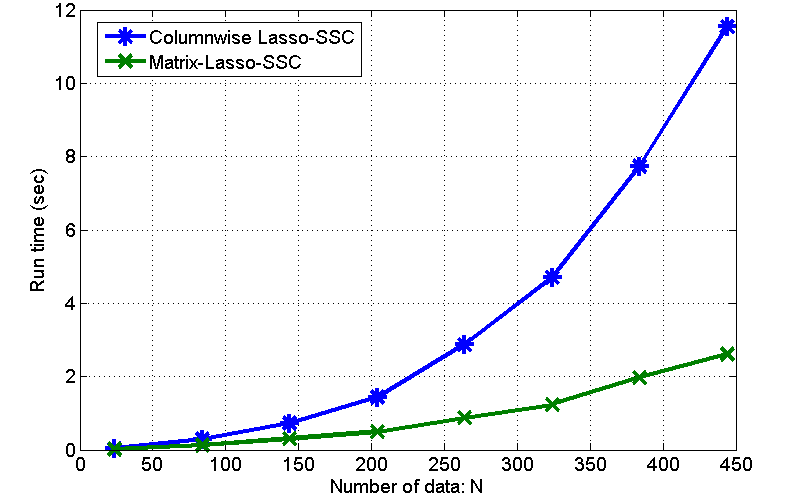}\\
  \caption{Run time comparison with increasing number of data. Simulated with $n=100, d=4, L=3, \sigma=0.2$, $\kappa$ increases from $2$ to $40$ such that the number of data goes from 24- 480. It appears that the matrix version scales better with increasing number of data compared to columnwise LASSO.}\label{fig.runtime_data}
\end{minipage}
\hspace{0.02\linewidth}
\begin{minipage}[t]{0.48\linewidth}
  \centering
  \includegraphics[width=8cm]{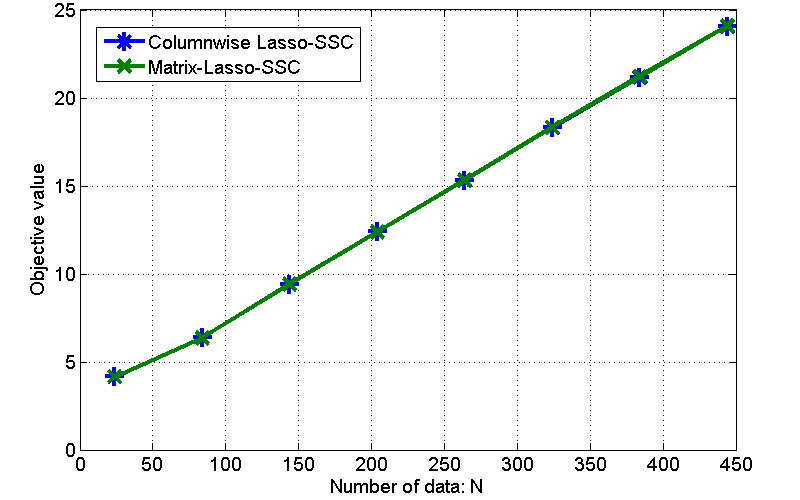}\\
  \caption{Objective value comparison with increasing number of data. Simulated with $n=100, d=4, L=3, \sigma=0.2$, $\kappa$ increases from $2$ to $40$ such that the number of data goes from 24- 480. The objective value obtained at stop points of two algorithms are nearly the same.}.\label{fig.obj_data}
\end{minipage}
\end{figure}

\begin{figure}[htb]
\begin{minipage}[t]{0.48\linewidth}
  \centering
  \includegraphics[width=8cm]{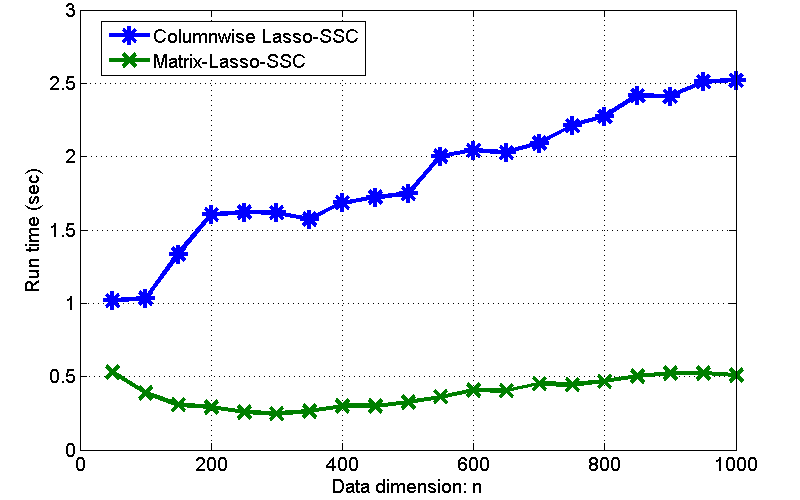}\\
  \caption{Run time comparison with increasing number of data. Simulated with $\kappa=5, d=4, L=3, \sigma=0.2$, ambient dimension $n$ increases from $50$ to $1000$. Note that the dependence on dimension is weak at the scale due to the fast vectorized computation. Nevertheless, it is clear that the matrix version of SSC runs faster.}\label{fig.runtime_dim}
\end{minipage}
\hspace{0.02\linewidth}
\begin{minipage}[t]{0.48\linewidth}
  \centering
  \includegraphics[width=8cm]{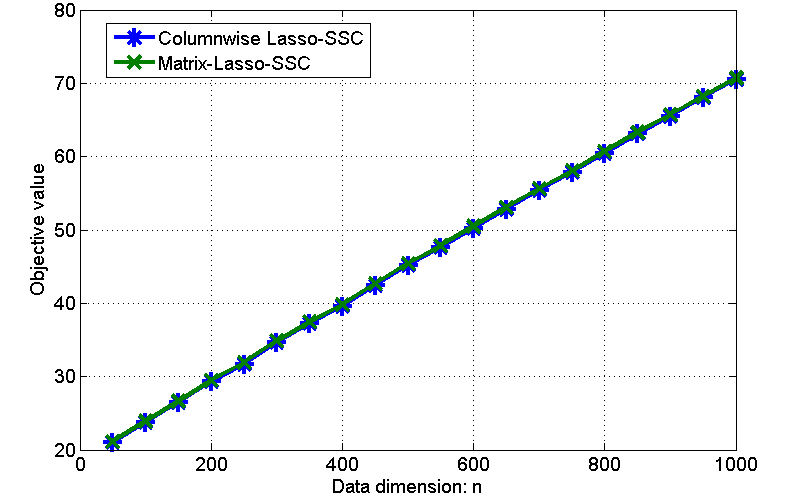}\\
  \caption{Objective value comparison with increasing number of data. Simulated with $\kappa=5, d=4, L=3, \sigma=0.2$, ambient dimension $n$ increases from $50$ to $1000$. The objective value obtained at stop points of two algorithms are nearly the same.}\label{fig.obj_dim}
\end{minipage}
\end{figure}

\vskip 0.2in
\bibliography{NoisySSC_ICML}
\end{document}